\documentclass[11pt,letterpaper]{article}

\usepackage[margin=1in]{geometry}
\usepackage{hyperref}
\usepackage[english]{babel}
\usepackage[utf8x]{inputenc}
\usepackage{amsmath}
\usepackage{amsfonts}
\usepackage{soul}
\usepackage{amssymb}
\usepackage{graphicx,subfigure}
\usepackage{rotating}
\usepackage{amsbsy}
\usepackage{graphicx,  ucs}
\usepackage{float}
\usepackage{tikz}
\usepackage{algorithm}
\usepackage[noend]{algpseudocode}
\usepackage{listings}
\usepackage{amsthm}
\usepackage{enumitem}
\usepackage{hyperref}
\usepackage{bm}
\usepackage[compact]{titlesec}

\usepackage{chngcntr}

\def\compactify{\itemsep=0pt \topsep=0pt \partopsep=0pt \parsep=0pt}
\let\latexusecounter=\usecounter

\newcommand{\Exp}{\mathbb{E}}
\newcommand{\tl}{\textlatin}

\def\poly{\mathrm{poly}}
\def\eps{\varepsilon}

\def\Prob{\mathbb{P}}
\def\Exp{\mathbb{E}}

\newcommand{\reals}{\mathbb{R}}

\def\norm#1{\left\|#1\right\|}
\def\normm#1{\left\|#1\right\|_{\Sigma}}
\def\abs#1{\left|#1\right|}

\newtheorem{theorem}{Theorem}
\newtheorem*{inftheorem}{Informal Theorem}
\newtheorem{proposition}{Proposition}
\newtheorem{remark}{Remark}

\newtheorem{lemma}{Lemma}
\newtheorem*{lemma25}{Lemma 25 of Daskalakis et al. 2016}
\newtheorem*{lemma2.1}{Lemma 2.1 of Wainright 2015}

\renewcommand{\vec}{\bm}
\renewcommand{\l}{\vec{\lambda}}
\renewcommand{\t}{\vec{\tau}}

\newcommand{\m}{\vec{\mu}}
\newcommand{\vx}{\vec{x}}
\newcommand{\vy}{\vec{y}}
\newcommand{\vc}{\vec{c}}
\renewcommand{\v}{\vec{v}}
\newcommand{\vd}{\vec{\delta}}

\definecolor{vergreen}{RGB}{0,85,2}
\definecolor{myvergreen}{RGB}{0,140,3}
\definecolor{provorange}{RGB}{85,34,0}
\definecolor{inputblue}{RGB}{5,13,111}
\definecolor{noapred}{RGB}{116,3,3}
\definecolor{classesblue}{RGB}{9,49,146}

\definecolor{secinhead}{RGB}{249,196,95}

\definecolor{lgray}{gray}{0.8}

\counterwithin*{equation}{section}

\newcommand{\extendedA}[2]{#1}
\begin{document}
\newcounter{tempcounter}

\begin{titlepage}
\title{Ten Steps of EM Suffice for Mixtures of Two Gaussians}
\date{}
\author{
Constantinos Daskalakis \\ EECS and CSAIL, MIT \\ \href{mailto:costis@mit.edu}{costis@mit.edu}
\and
  Christos Tzamos \\ EECS and CSAIL, MIT \\ \href{mailto:tzamos@mit.edu}{tzamos@mit.edu}
	\and Manolis Zampetakis \\ EECS and CSAIL, MIT \\ \href{mailto:mzampet@mit.edu}{mzampet@mit.edu} 
	}
\clearpage
\maketitle
\thispagestyle{empty}
\begin{abstract}
The Expectation-Maximization (EM) algorithm is a widely used method for maximum likelihood estimation in models with latent variables. For estimating mixtures of Gaussians, its iteration can be viewed as a soft version of the k-means clustering algorithm. Despite its wide use and applications, there are essentially no known convergence guarantees for this method. We provide global convergence guarantees for  mixtures of two Gaussians with known covariance matrices. We show that the population version of EM, where the algorithm is given access to infinitely many samples from the mixture, converges geometrically to the correct mean vectors, and provide simple, closed-form expressions for the convergence rate. As a simple illustration, we show that, in one dimension, ten steps of the EM algorithm initialized at infinity result in less than 1\% error estimation of the means. In the finite sample regime, we show that, under a random initialization, $\tilde{O}(d/\epsilon^2)$ samples suffice to compute the unknown vectors to within $\epsilon$ in Mahalanobis distance, where $d$ is the dimension. In particular, the error rate of the EM based estimator is $\tilde{O}\left(\sqrt{d \over n}\right)$ where $n$ is the number of samples, which is optimal up to logarithmic factors.
\end{abstract}
\end{titlepage}

  \section{Introduction}
\label{sec:intro}

The {\em Expectation-Maximization (EM) algorithm}~\cite{DempsterLR77,Wu83,RednerW84} is one of the most widely used heuristics for maximizing likelihood in statistical models with latent variables. Consider a probability distribution $p_{\vec{\lambda}}$ sampling $(\vec{X}, \vec{Z})$, where $\vec{X}$ is a vector of observable random variables, $\vec{Z}$  a vector of non-observable random variables and $\vec{\lambda} \in \vec{\Lambda}$  a vector of parameters. Given independent samples $\vec{x}_1,\ldots,\vec{x}_n$ of the observed random variables, the goal of maximum likelihood estimation is to select $\vec{\lambda} \in \vec{\Lambda}$ maximizing the log-likelihood of the samples, namely $\sum_i \log p_{\vec{\lambda}}(\vec{x}_i)$. Unfortunately, computing $p_{\vec{\lambda}}(\vec{x}_i)$ involves  summing $p_{\vec{\lambda}}(\vec{x}_i,\vec{z}_i)$ over all possible values of $\vec{z}_i$, which commonly results in a log-likelihood function that is non-convex with respect to $\vec{\lambda}$ and therefore hard to optimize. In this context, the EM algorithm proposes the following heuristic:
\begin{itemize}
\item Start with an initial guess $\vec{\lambda}^{(0)}$ of the parameters. 
\item For all $t \ge 0$, until convergence:
\begin{itemize}
\item (E-Step) For each sample $i$, compute the posterior $Q^{(t)}_i(\vec{z}):=p_{\vec{\lambda^{(t)}}}(\vec{Z}=\vec{z}| \vec{X}=\vec{x}_i)$.

\item (M-Step) Set $\vec{\lambda^{(t+1)}} := \arg \max_{\vec{\lambda}} \sum_i \sum_{\vec{z}} Q_i^{(t)}(\vec{z}) \log {p_{\vec{\lambda}}(\vec{x}_i, \vec{z}) \over Q_i^{(t)}(\vec{z})}$.
\end{itemize}
\end{itemize}

Intuitively, the E-step of the algorithm uses the current guess of the parameters, $\vec{\lambda}^{(t)}$, to form beliefs, $Q^{(t)}_i$,  about the state of the (non-observable) $\vec{Z}$ variables for each sample $i$. Then the M-step uses the new beliefs about the state of $\vec{Z}$ for each sample to maximize with respect to $\vec{\lambda}$ a lower bound on $\sum_i \log p_{\vec{\lambda}}(\vec{x}_i)$. Indeed, by the concavity of the $\log$ function, the objective function used in the M-step of the algorithm is a lower bound on the true log-likelihood for all values of $\vec{\lambda}$, and it equals the true log-likelihood for $\vec{\lambda}=\vec{\lambda}^{(t)}$. From these observations, it follows that the above alternating procedure improves the true log-likelihood until convergence. 

Despite its wide use and practical significance, little is known about whether and under what conditions EM converges to the true maximum likelihood estimator. A few works establish local convergence of the algorithm to stationary points of the log-likelihood function~\cite{Wu83,Tseng04,ChretienH08}, and even fewer local convergence to the MLE~\cite{RednerW84,BalakrishnanWW14}. Besides local convergence, it is also known that badly initialized EM may settle far from the MLE both in parameter and in likelihood distance~\cite{Wu83}. The lack of theoretical understanding of the convergence properties of EM is intimately related to the non-convex nature of the optimization it performs. 

{ Our paper aims to illuminate why EM works well in practice and develop techniques for understanding its behavior. We do so by analyzing one of the  most basic and natural, yet still challenging, statistical models EM may be applied to, namely balanced mixtures of two multi-dimensional Gaussians with equal and known covariance matrices.} In particular, we study the convergence of EM when applied to the following family of parametrized density functions:
$$p_{\vec{\mu}_1,\vec{\mu}_2}(\vec{x}) = 0.5 \cdot {\cal N}(\vec{x} ; \vec{\mu}_1, \Sigma) + 0.5 \cdot {\cal N}(\vec{x} ; \vec{\mu}_2, \Sigma),$$
 where $\Sigma$ is a known covariance matrix, $(\vec{\mu}_1,\vec{\mu}_2)$ are unknown (vector) parameters, and $\mathcal{N}(\vec{\mu}, \Sigma; \vec{x})$ represents the Gaussian density with mean $\vec{\mu}$ and covariance matrix $\Sigma$, i.e. 
\[ \mathcal{N}(\vec{x};\vec{\mu}, \Sigma) = \frac{1}{\sqrt{2 \pi \det{\Sigma}}} \exp{\left( - 0.5 (\vec{x} - \vec{\mu})^T \Sigma^{-1} (\vec{x} - \vec{\mu}) \right)}.\]

Our main contribution is to  provide global convergence guarantees for EM applied to the above family of distributions. We establish our result for both the  ``population version'' of the algorithm, and the finite-sample version, as described below. 

\paragraph{Analysis of Population EM for Mixtures of Two Gaussians.} To elucidate the optimization features of the algorithm and avoid analytical distractions arising due to sampling error, it has been standard practice in the literature of theoretical analyses of EM to consider the ``population version'' of the algorithm, where the EM iterations are performed assuming access to infinitely many samples from a distribution $p_{\vec{\mu}_1,\vec{\mu}_2}$ as above. With infinitely many samples, we can identify the mean, ${\vec{\mu}_1 +\vec{\mu}_2 \over 2}$, of $p_{\vec{\mu}_1,\vec{\mu}_2}$, and re-parametrize the density around the mean as follows:
\begin{align}
p_{\vec{\mu}}(\vec{x}) = 0.5 \cdot {\cal N}(\vec{x} ; \vec{\mu}, \Sigma) + 0.5 \cdot {\cal N}(\vec{x} ; -\vec{\mu}, \Sigma). \label{eq:symmetrized density}
\end{align}

We first study the convergence of EM when we perform iterations with respect to the parameter $\vec{\mu}$ of $p_{\vec{\mu}}(\vec{x})$ in~\eqref{eq:symmetrized density}. Starting with an initial guess $\l^{(0)}$ for the unknown mean vector $\vec{\mu}$, the $t$-th iteration of EM amounts to the following update:
\begin{equation}
\vec{\lambda}^{(t + 1)} = M(\vec{\lambda}^{(t)}, \vec{\mu}) \triangleq \frac{\Exp_{\vec{x} \sim p_{\m}}\left[ \frac{0.5 \mathcal{N}(\vec{x};\l^{(t)}, \Sigma )}{p_{\l^{(t)}}(\vec{x})} \vec{x} \right]}{\Exp_{\vec{x} \sim p_{\m}}\left[ \frac{0.5\mathcal{N}(\vec{x};\l^{(t)}, \Sigma)}{p_{\l^{(t)}}(\vec{x})} \right]}, \label{eq:EM update raw}
\end{equation}
where we have compacted both the E- and M-step of EM into one update. 

The intuition behind the EM update formula is as follows. First, we take expectations with respect to $\vec{x} \sim p_{\m}$ because we are studying the population version of EM, hence we assume access to infinitely many samples from $p_{\m}$. For each sample $\vec{x}$, the ratio $\frac{0.5 \mathcal{N}(\vec{x};\l^{(t)}, \Sigma )}{p_{\l^{(t)}}(\vec{x})}$ is our belief, at step $t$, that $\vec{x}$ was sampled from the first Gaussian component of $p_{\m}$, namely the one for which our current estimate of its mean vector is $\l^{(t)}$. (The complementary probability is our present belief that $\vec{x}$ was sampled from the other Gaussian component.) Given these beliefs for all vectors $\vec{x}$, the update~\eqref{eq:EM update raw} is the result of the M-step of EM. Intuitively, our next guess $\l^{(t+1)}$ for the mean vector of the first Gaussian component is a weighted combination over all samples $\vec{x} \sim p_{\m}$ where the weight of every $\vec{x}$ is our belief that it came from the first Gaussian component. 

Our main result for population-EM is the following:

\begin{inftheorem}[Population EM Analysis] \label{infthm:population} Whenever the initial guess $\l^{(0)}$ is not equidistant to $\m$ and $-\m$, EM converges geometrically to either $\m$ or $-\m$, with convergence rate that improves as $t \rightarrow \infty$. We provide a simple, closed form expression of the convergence rate as a function of $\l^{(t)}$ and $\m$. If the initial guess $\l^{(0)}$ is equidistant to $\m$ and $-\m$, EM converges to the unstable fixed point $\bf 0$.
\end{inftheorem}

\extendedA{A formal statement is provided as Theorem~\ref{thm:multid} in Section~\ref{sec:multid}. We start with the proof of the single-dimensional version, presented as Theorem~\ref{thm:singled} in Section~\ref{sec:singled}. As a simple illustration of our result, we show in Section~\ref{sec:quan} that, in one dimension, when our original guess $\lambda^{(0)} = +\infty$ and the signal-to-noise ratio $\mu/\sigma=1$, $10$ steps of the EM algorithm result in $1\%$ error.}{A formal statement of Informal Theorem \ref{infthm:population} is provided in the full version of the paper. As a simple illustration of our result, we show that, in one dimension, when our original guess $\lambda^{(0)} = +\infty$ and the signal-to-noise ratio $\mu/\sigma=1$, $10$ steps of the EM algorithm result in $1\%$ error.}

Despite the simplicity of the case we consider, no global convergence results were known prior to our work, even for the population EM. \cite{BalakrishnanWW14} studied the same setting proving only local convergence, i.e. convergence only when the initial guess is close to the true parameters. They argue that the population EM update is contracting close to the true parameters. Unfortunately, the EM update is non-contracting outside a small neighborhood of the true parameters so this argument cannot be used for a global convergence guarantee.

 In this work, we study the problem under arbitrary starting points and completely characterize the fixed points of EM. We show that other than a measure-zero subset of the space (namely points that are equidistant from the centers of the two Gaussians), any initialization of the EM algorithm converges  to the true centers of the Gaussians, providing explicit bounds for the convergence rate. To achieve this, we follow an orthogonal approach to \cite{BalakrishnanWW14}: Instead of trying to directly compute the number of steps required to reach convergence {\em for a specific instance of the problem}, we study {\em the sensitivity of the EM iteration as the instance varies}. The intuition is that if the EM update is sensitive to updating the instance, then changing the instance should also attract the update towards the changing instance; see Figure~\ref{fig:sensitivity}. We can use this, in turn, to argue that keeping the instance fixed, one EM update makes progress towards the true parameters. In particular, we gain a handle on the convergence rate of EM on  all instances at once.\extendedA{ This is quantified by Eq.~\eqref{eq: sensitivity wrt instance}.}{}

\begin{figure}[!h]
	\centering
		\includegraphics[scale=0.43]{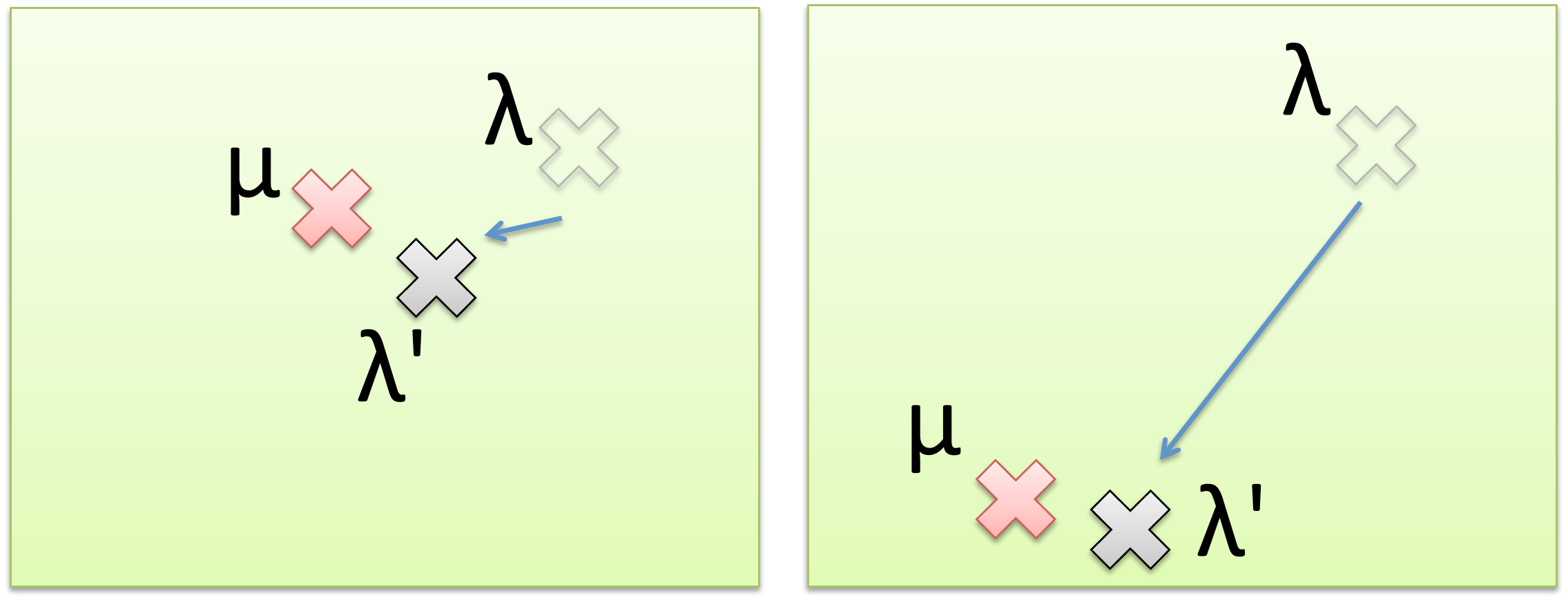}
	\caption{Sensitivity of the EM update when changing the true parameters. Large sensitivity implies large progress towards the true parameters.\extendedA{ See Eq.~\eqref{eq: sensitivity wrt instance} for a quantification.}{}}
	\label{fig:sensitivity}
\end{figure}

\paragraph{Analysis of Finite-Sample EM for Mixtures of Two Gaussians.}  The finite sample analysis proceeds in three steps. First, in the finite sample regime we do not know the average of the two mean vectors, $(\vec{\mu}_1 + \vec{\mu}_2)/2$, exactly. We show that, with $\tilde{O}(d/\epsilon^2)$ samples, we can approximate the average to within Mahalanobis distance $\epsilon$. We then chain two coupling arguments. The first compares the progress towards the true mean made by the correctly centered population EM update to that of the incorrectly centered population EM update. The second compares the progress towards the true mean made by the incorrectly centered population EM update with the progress made by the incorrectly centered finite sample EM update. \extendedA{See Figure~\ref{fig:graph2} and Theorem~\ref{thm:sampleBased}}{See Figure~\ref{fig:graph2}}. 
Given the error incurred in the approximation of the center $(\vec{\mu}_1 + \vec{\mu}_2)/2$, we propose to stabilize the sample-based EM iteration by including in the sample for each sampled point $\vx_i$ its symmetric point $-\vx_i$. This is the sample based version that we analyze, although our analysis goes through without this stabilization. \extendedA{Our result is the following, formally given as Theorem~\ref{thm:sampleBased} in Section~\ref{sec:sample}}{Our result is the following, formally given in the full version}.

\begin{inftheorem}[Finite Sample EM Analysis] Whenever $\epsilon < SNR$, $\tilde{O}(d/\epsilon^2 \cdot {\rm poly}(1/SNR))$ samples suffice to approximate $\m_1$ and $\m_2$ to within Mahalanobis distance $\epsilon$ using the EM algorithm. In particular, the error rate of the EM based estimator is $\tilde{O}\left(\sqrt{d \over n}\right)$ where $n$ is the number of samples, which is optimal up to logarithmic factors.\footnote{Note that even if $SNR$ is arbitrarily large (so that the two Gaussian components are ``perfectly separated'') the problem degenerates to finding the mean of one Gaussian whose optimal rate is $\Omega\left(\sqrt{d \over n}\right)$.}
\end{inftheorem}

\paragraph{Bootstrapping EM for Faster Convergence.} We note that, in multiple dimensions, care must be taken in initializing the EM algorithm, even in the infinite sample regime, as the convergence guarantee depends on the angle between the current iterate and the true mean vector. While a randomly chosen unit vector will have projection of $\Theta(1/\sqrt{d})$ in the direction of $\m$, we argue that we can boostrap EM to turn this projection larger than a constant. This allows us to work with similar convergence rates as in the single-dimensional case, namely only SNR (and not dimension) dependent.\extendedA{ Our initialization procedure is described in Section~\ref{sec:initialization}.}{}

\begin{inftheorem}[EM Initialization] EM can be boostrapped so that the number of iterations required to approximate $\m_1$ and $\m_2$ to within Mahalanobis distance $\epsilon$ depends logarithmically in the dimension.
\end{inftheorem}

\begin{figure}
	\centering
  \subfigure[]{\includegraphics[scale=0.4]{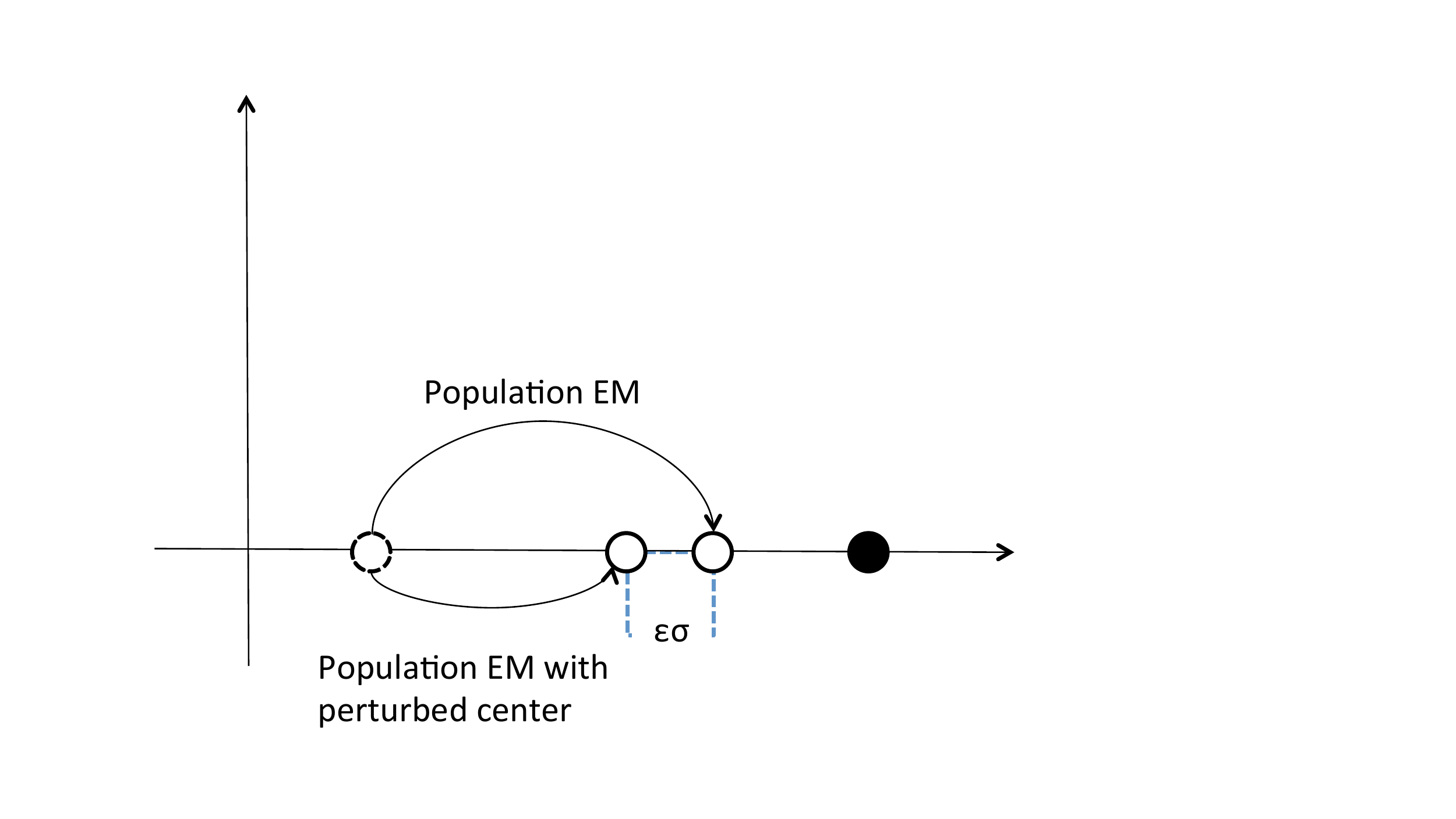}} ~~~
  \subfigure[]{\includegraphics[scale=0.4]{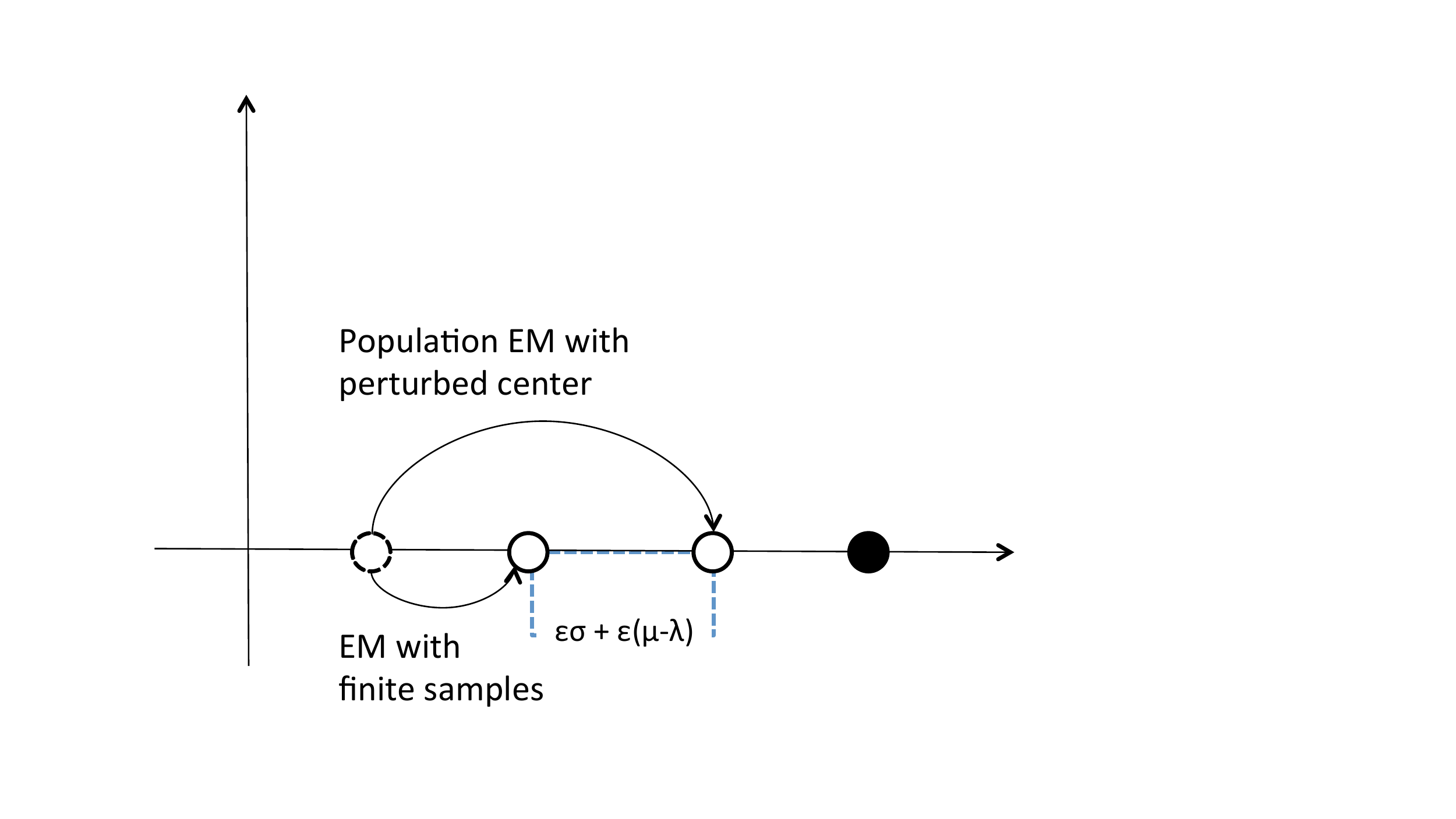}}

  \caption{(a) Coupling  correctly and incorrectly centered population EM updates. We show that, starting from the same iterate, the correctly and incorrectly centered population EM updates will land to close-by points.\extendedA{ This quantified by Eq.~\eqref{eq:SBEMpopulContr3}.}{} (b) Coupling  incorrectly centered population EM and finite sample EM updates. We show that, starting from the same iterate, the incorrectly centered population EM update and the finite sample update land to close-by points.\extendedA{ This quantified by Lemma~\ref{lem:basicConcentration}.}{}}
  \label{fig:graph2}
\end{figure}

\paragraph{Related Work on Learning Mixtures of Gaussians.}

We have already outlined the literature on the Expectation-Maximization algorithm. Several results study its local convergence properties and there are known cases where badly initialized EM fails to converge. See above.

There is also a large body of literature on learning mixtures of Gaussians.  A long line of work initiated by Dasgupta~\cite{dasgupta1999learning, sanjeev2001learning, vempala2004spectral, achlioptas2005spectral,kannan2005spectral,DasguptaS07,chaudhuri2008learning,brubaker2008isotropic,chaudhuri2009learning} provides rigorous guarantees on recovering the parameters of Gaussians in a mixture under separability assumptions, while later work~\cite{kalai2010efficiently, moitra2010settling, belkin2010polynomial} has established guarantees under minimal information theoretic assumptions. More recent work~\cite{hardt2015tight} provides tight bounds on the number of samples necessary to recover the parameters of the Gaussians as well as improved algorithms, while another strand of the literature studies proper learning with improved running times and sample sizes~\cite{suresh2014near,daskalakis2014faster}. Finally, there has been work on methods exploiting general position assumptions or performing smoothed analysis~\cite{HsuK13,GeHK15}.

In practice, the most common algorithm for learning mixtures of Gaussians is the Expectation-Maximization algorithm, with the practical experience that it performs well in a broad range of scenarios despite the lack of theoretical guarantees. Recently, Balakrishnan, Wainwright and Yu~\cite{BalakrishnanWW14} studied the convergence of EM in the case of an equal-weight mixture of two Gaussians with the same and known covariance matrix, showing local convergence guarantees. In particular, they show that when EM is initialized close to the actual parameters, then it converges. In this work, we revisit the same setting considered by~\cite{BalakrishnanWW14} but establish {\em global convergence guarantees}. We show that, for any initialization of the parameters, the EM algorithm converges geometrically to the true parameters. We also provide a simple and explicit formula for the rate of convergence. 

Concurrent and independent work by Xu, Hsu and Maleki~\cite{xu2016global} has also provided global and geometric convergence guarantees for the same setting, as well as a slightly more general setting where the mean of the mixture is unknown, but they do not provide explicit convergence rates. They also do not provide an analysis of the finite-sample regime. 
  \section{Preliminary Observations}
\label{sec:model}

In this section we illustrate some simple properties of the EM update~\eqref{eq:EM update raw} and simplify the formula. First, it is easy to see that plugging in the values $\l \in \{ -\m, \vec{0}, \m\}$ into $M(\l, \m)$ results into
\begin{equation}
    \label{eq:values}
    M(-\m,\m)=-\m \quad; \quad M(\vec{0},\m)=\vec{0} \quad; \quad M(\m,\m) = \m.
\end{equation}
In particular, for all $\m$, these values are certainly fixed points of the EM iteration. Next, we  rewrite $M(\l,\m)$ as follows:
\[ M(\l, \m) = 
\frac {
  \frac{1}{2} \Exp_{\vec{x} \sim \mathcal{N}(\m, \Sigma)}\left[ \frac{0.5 \mathcal{N}(\vec{x};\l, \Sigma)}{p_{\l}(\vec{x})} \vec{x} \right] +
  \frac{1}{2} \Exp_{\vec{x} \sim \mathcal{N}(-\m, \Sigma)}\left[ \frac{0.5 \mathcal{N}(\vec{x};\l, \Sigma)}{p_{\l}(\vec{x})} \vec{x} \right]
} {
  \frac{1}{2} \Exp_{\vec{x} \sim \mathcal{N}(\m, \Sigma)}\left[ \frac{0.5 \mathcal{N}(\vec{x};\l, \Sigma)}{p_{\l}(\vec{x})} \right] +
  \frac{1}{2} \Exp_{\vec{x} \sim \mathcal{N}(-\m, \Sigma)}\left[ \frac{0.5 \mathcal{N}(\vec{x};\l, \Sigma)}{p_{\l}(\vec{x})} \right]
}. 
\]
It is easy to observe that by symmetry this simplifies to
\[ M(\l, \m) = 
\frac {
  \frac{1}{2} \Exp_{\vec{x} \sim \mathcal{N}(\m, \Sigma)}\left[ 
    \frac {
      \frac{1}{2} \mathcal{N}(\vec{x};\l, \Sigma) - \frac{1}{2} \mathcal{N}(\vec{x};-\l, \Sigma)
    }{
      \frac{1}{2} \mathcal{N}(\vec{x};\l, \Sigma) + \frac{1}{2} \mathcal{N}(\vec{x};-\l, \Sigma)
    } \vec{x} 
  \right]
} {
  \frac{1}{2} \Exp_{\vec{x} \sim \mathcal{N}(\m, \Sigma)}\left[ 
    \frac {
      \frac{1}{2} \mathcal{N}(\vec{x};\l, \Sigma) + \frac{1}{2} \mathcal{N}(\vec{x};-\l, \Sigma)
    }{
      \frac{1}{2} \mathcal{N}(\vec{x};\l, \Sigma) + \frac{1}{2} \mathcal{N}(\vec{x};-\l, \Sigma)
    }
  \right]
} =
  \Exp_{\vec{x} \sim \mathcal{N}(\m, \Sigma)}\left[ 
    \frac {
      \mathcal{N}(\vec{x};\l, \Sigma) - \mathcal{N}(\vec{x};-\l, \Sigma)
    }{
      \mathcal{N}(\vec{x};\l, \Sigma) + \mathcal{N}(\vec{x};-\l, \Sigma)
    } \vec{x} 
  \right].
\]
Simplifying common terms in the density functions $\mathcal{N}(\vec{x};\l, \Sigma)$, we get that
\[ M(\l, \m) = \Exp_{\vec{x} \sim \mathcal{N}(\m, \Sigma)}\left[ 
\frac{
  \exp{\left( \l^T \Sigma^{-1} \vec{x} \right)} - \exp{ \left( - \l^T \Sigma^{-1} \vec{x} \right)}
}{
  \exp{\left( \l^T \Sigma^{-1} \vec{x} \right)} + \exp{\left( - \l^T \Sigma^{-1} \vec{x} \right)}
} \vec{x} \right]. \]
We thus get the following expression for the EM iteration
\begin{equation}\label{eq:EMiteration}
M(\vec{\lambda}, \vec{\mu}) = \Exp_{\vec{x} \sim \mathcal{N}(\m, \Sigma)}\left[ \tanh( \l^T \Sigma^{-1} \vec{x} ) \vec{x} \right].
\end{equation}

  \section{Single-dimensional Convergence}
\label{sec:singled}

  In the single dimensional case the EM algorithm takes the following form according to (\ref{eq:EMiteration}).
\begin{equation}
\label{eq:EMiteration1D}
\lambda^{(t + 1)} = M(\lambda^{(t)}, \mu) = \Exp_{x \sim \mathcal{N}(\mu, \sigma^2)}\left[ \tanh\left( \frac{\lambda^{(t)} x}{\sigma^2} \right) x \right]
\end{equation}

  Observe that the function $M(\lambda, \mu)$ is increasing with respect to $\lambda$. Indeed the partial derivative of $M$ with respect to $\lambda$ is
  \[ \frac{\partial M(\lambda, \mu)}{\partial \lambda} = \Exp_{x \sim \mathcal{N}(\mu, \sigma^2)}\left[ \tanh'\left( \frac{\lambda^{(t)} x}{\sigma^2} \right) \frac{x^2}{\sigma^2} \right] \]
\noindent which is strictly greater than zero since the $\tanh'$ function is strictly positive. 

We will show next that the fixed points we identified at (\ref{eq:values}) are the only fixed points of $M(\cdot, \mu)$. When initialized with $\lambda^{(0)} > 0$ (resp. $\lambda^{(0)} < 0$), the EM algorithm converges to $\mu > 0$ 
(resp. to $-\mu < 0$). The point $\lambda = 0$ is an unstable fixed point.

\begin{theorem}
\label{thm:singled}
  In the single dimensional case, when $\lambda^{(0)}, \mu > 0$, the parameters $\lambda^{(t)}$ satisfy
  \[ \abs{\lambda^{(t + 1)} - \mu} \le \kappa^{(t)} \abs{\lambda^{(t)} - \mu} 
\quad \text{ where }
 \kappa^{(t)} = \exp\left( - \frac{ \min(\lambda^{(t)}, \mu)^2 }{2 \sigma^2} \right) \]
Moreover $\kappa^{(t)}$ is a decreasing function of $t$.
\end{theorem}

\begin{proof}
  For simplicity we will use $\lambda$ for $\lambda^{(t)}$, $\lambda'$ for $\lambda^{(t + 1)}$ and we will assume that $X \sim \mathcal{N}(0, \sigma^2)$.

  By a simple change of variables we can see that  
\[ M(\lambda, \mu) = \Exp\left[ \tanh\left( \frac{\lambda (X + \mu)}{\sigma^2} \right) (X + \mu) \right] \]

The main idea is to use the Mean Value Theorem with respect to the second coordinate of the function $M$ on the interval $[\lambda, \mu]$.
\[ \frac{M(\lambda, \mu) - M(\lambda, \lambda)}{\mu - \lambda} = \left. \frac{\partial M(\lambda, y)}{\partial y} \right|_{y = \xi} \text{ with } \xi \in (\lambda, \mu) \]

But we know that $M(\lambda, \lambda) = \lambda$ and $M(\lambda, \mu) = \lambda'$ and therefore we get
\[ \lambda' - \lambda \ge \left( \min_{\xi \in [\lambda, \mu]} \left. \frac{\partial M(\lambda, y)}{\partial y} \right|_{y = \xi} \right) (\mu - \lambda) \]
which is equivalent to
\begin{align} \abs{\lambda' - \mu} \le \left( 1 - \min_{\xi \in [\lambda, \mu]} \left. \frac{\partial M(\lambda, y)}{\partial y} \right|_{y = \xi} \right) \abs{\lambda - \mu} \label{eq: sensitivity wrt instance} 
\end{align}
where we have used the fact that $\lambda' < \mu$ which is comes from the fact that $M(\lambda, \mu)$ is increasing with respect to $\lambda$ and that $M(\mu, \mu) = \mu$.

The only thing that remains to complete our proof is to prove a lower bound of the partial derivative of $M$ with respect to $\mu$.
\[ \left. \frac{\partial M(\lambda, y)}{\partial y} \right|_{y = \xi} = \Exp \left[ \frac{\lambda}{\sigma^2} \tanh'\left( \frac{\lambda (X + \xi)}{\sigma^2} \right) (X + \xi) + \tanh \left( \frac{\lambda (X + \xi)}{\sigma^2} \right) \right] \]

The first term is non-negative, Lemma \ref{lem:tanhdotlemma}. The second term is at least $1 - \exp \left[ - \frac{ \min(\xi, \lambda) \cdot \xi }{ 2 \sigma^2} \right]$, Lemma \ref{lem:tanhlemma} and the theorem follows.
\end{proof}

\begin{lemma}
  \label{lem:tanhdotlemma}
  Let $\alpha, \beta > 0$ and $X \sim \mathcal{N}(\alpha , \sigma^2)$ then $\Exp \left[ \tanh ' \left( {\beta X / \sigma^2} \right) X \right] \ge 0$.
\end{lemma}

\begin{proof}
\[ \Exp \left[ \tanh'\left( \frac{\beta X}{\sigma^2} \right) X \right] = \frac{1}{\sqrt{2 \pi} \sigma} \int_{- \infty}^{\infty} \tanh'\left( \frac{\beta y}{\sigma^2} \right) y \exp{\left( - \frac{(y - \alpha)^2}{2 \sigma^2} \right)} dy \]
But now we can see that since $\tanh'$ is an even function and since for any $y > 0$ we have $\exp{\left( - \frac{(y - \alpha)^2}{2 \sigma^2} \right)} \ge \exp{\left( - \frac{(- y - \alpha)^2}{2 \sigma^2} \right)}$ then
\[ - \frac{1}{\sqrt{2 \pi} \sigma} \int_{- \infty}^{0} \tanh'\left( \frac{\beta y}{\sigma^2} \right) y \exp{\left( - \frac{(y - \alpha)^2}{2 \sigma^2} \right)} dy \le \frac{1}{\sqrt{2 \pi} \sigma} \int_{0}^{\infty} \tanh'\left( \frac{\beta y}{\sigma^2} \right) y \exp{\left( - \frac{(y - \alpha)^2}{2 \sigma^2} \right)} dy \]
which means that
$ \Exp \left[ \tanh ' \left( {\beta X / \sigma^2} \right) X \right] \ge 0 $.
\end{proof}

\begin{lemma}
  \label{lem:tanhlemma}
  Let $\alpha, \beta > 0$ and $X \sim \mathcal{N}(\alpha , \sigma^2)$ then $\Exp \left[ \tanh \left( {\beta X / \sigma^2} \right) \right] \ge 1 - \exp \left[ - \frac{ \min(\alpha,\beta) \cdot \alpha }{ 2 \sigma^2} \right]$.
\end{lemma}

\begin{proof}
Note that $\Exp \left[ \tanh \left( {\beta X / \sigma^2} \right) \right]$ is increasing as a function of $\beta$ as its derivative with respect to $\beta$ is positive by Lemma~\ref{lem:tanhdotlemma}. It thus suffices to show that $\Exp \left[ \tanh \left( {\beta X / \sigma^2} \right) \right] \ge 1 - \exp \left[ - \frac{ \alpha \beta  }{ 2 \sigma^2} \right]$ when $\beta \le \alpha$. We have that
\begin{align*}
  \Exp &\left[ 1 - \tanh \left( {\beta X / \sigma^2} \right) \right] = \Exp \left[ \frac { 2 } { \exp( {2 \beta X / \sigma^2} ) + 1 } \right] \le \Exp \left[ \frac { 1 } { \exp( { \beta X / \sigma^2} )  } \right]  \\
  &= \frac{1}{\sqrt{2 \pi} \sigma}
\int_{-\infty}^\infty \frac{ \exp \left( - \frac{(x - \alpha)^2}{2 \sigma^2} \right) } {\exp( { \beta x / \sigma^2)}} dx = \frac{\exp \left( \frac{(\alpha - \beta)^2 - \alpha^2}{2 \sigma^2} \right)}{\sqrt{2 \pi} \sigma}
\int_{-\infty}^\infty \exp \left( - \frac{(x - \alpha + \beta)^2}{2 \sigma^2} \right) dx \\
  &= \exp \left( \frac{(\alpha - \beta)^2 - \alpha^2}{2 \sigma^2} \right) \le \exp \left( - \frac{\alpha \beta}{2 \sigma^2} \right)
\end{align*}
which completes the proof.
\end{proof}

  \section{Multi-dimensional Convergence}
\label{sec:multid}

  In the multidimensional case, the EM algorithm takes the form of (\ref{eq:EMiteration}). In this case, we will quantify our approximation guarantees using the {\em Mahalanobis distance} $\norm{\cdot}_{\Sigma}$ between vectors with respect to matrix $\Sigma$, defined as follows:
  \[ \norm{\vec{x} - \vec{y}}_{\Sigma} = \sqrt{(\vec{x} - \vec{y})^T \Sigma^{-1} (\vec{x} - \vec{y})}. \]

  We will show that the fixed points identified in (\ref{eq:values}) are the only fixed points of $M(\cdot, \m)$. When initialized with $\l^{(0)}$ such that $\normm{\l^{(0)} - \m} < \normm{\l^{(0)} + \m}$ 
(resp. $\normm{\l^{(0)} - \m} > \normm{\l^{(0)} + \m}$), the EM algorithm converges to $\m$ (resp. to $-\m$). The algorithm converges to $\l = \vec{0}$ when initialized with $\normm{\l^{(0)} - \m} = \normm{\l^{(0)} + \m}$. In particular,

\begin{theorem}\label{thm:multid}
  Whenever $\normm{\l^{(0)} - \m} < \normm{\l^{(0)} + \m}$, i.e. the initial guess is closer to $\m$ than $-\m$, the estimates $\l^{(t)}$ of the EM algorithm satisfy 
  \[ \norm{\l^{(t + 1)} - \m}_{\Sigma} \le \kappa^{(t)} \norm{\l^{(t)} - \m}_{\Sigma}, \quad
\text{ where }
 \kappa^{(t)} = \exp\left( - \frac { \min\left(\l^{(t),T} \Sigma^{-1} \l^{(t)}, \m^T \Sigma^{-1} \l^{(t)}\right)^2 } { 2 \l^{(t),T} \Sigma^{-1} \l^{(t)}} \right).  \]
Moreover, $\kappa^{(t)}$ is a decreasing function of $t$. The symmetric things hold when $\normm{\l^{(0)} - \m} > \normm{\l^{(0)} + \m}$. When the initial guess is equidistant to $\m$ and $-\m$, then $\l^{(t)}=\vec{0}$ for all $t>0$.
\end{theorem}

\begin{proof}
  For simplicity we will use $\l$ for $\l^{(t)}$, $\l'$ for $\l^{(t + 1)}$. 
  
  By applying the following change of variables $\l \leftarrow \Sigma^{-1/2} \l$ and $\m \leftarrow \Sigma^{-1/2} \m$ we may assume that $\Sigma = I$ where $I$ is the identity matrix. Therefore the iteration of EM becomes
  \[ M(\vec{\lambda}, \vec{\mu}) = \Exp_{\vec{x} \sim \mathcal{N}(\m, I)}\left[ \tanh( \langle \l, \vec{x} \rangle ) \vec{x} \right] = \Exp_{\vec{x} \sim \mathcal{N}(\vec{0}, I)}\left[ \tanh( \langle \l, \vec{x} \rangle + \langle \l, \m \rangle ) (\vec{x} + \m) \right] \]

  Let $\hat{\l}$ be the unit vector in the direction of $\l$, $\hat{\l}^{\bot}$ be the unit vector that belongs to the plane of $\m, \l$ and is perpendicular to $\l$, and let $\{ \vec{v}_1=\hat{\l},\vec{v}_2=\hat{\l}^{\bot},\vec{v}_3,...,\vec{v}_d \}$ be a basis of $\mathbb{R}^d$. We have:
\begin{equation}
      \langle \vec{ v_i }, \l' \rangle = \Exp_{\vec{x} \sim \mathcal{N}(\vec{0}, I)}\left[ \tanh( \langle \l, \vec{x} \rangle + \langle \l, \m \rangle ) (\langle \vec{ v_i }, \vec{x} \rangle + \langle \vec{ v_i }, \m \rangle ) \right]
\end{equation}
Since the Normal distribution is rotation invariant we can equivalently write:
$$
      \langle \vec{ v_i }, \l' \rangle = \Exp_{\alpha_1,...,\alpha_d \sim \mathcal{N}(0, 1)}\left[ \tanh( \langle \l, \sum_j \alpha_j \vec{v}_j \rangle + \langle \l, \m \rangle ) (\langle \vec{ v_i }, \sum_j \alpha_j \vec{v}_j \rangle + \langle \vec{ v_i }, \m \rangle ) \right]
$$
which simplifies to
$$
\langle \vec{ v_i }, \l' \rangle = \Exp_{\alpha_1,...,\alpha_d \sim \mathcal{N}(0, 1)}\left[ \tanh( \alpha_1 \norm{\l} + \langle \l, \m \rangle ) (a_i + \langle \vec{ v_i }, \m \rangle  ) \right] =
$$
\begin{equation}
    \label{eq:EMprojection}
    \Exp_{\alpha_1 \sim \mathcal{N}(0, 1)}\left[ \tanh( \alpha_1 \norm{\l} + \langle \l, \m \rangle ) \cdot  (\Exp_{\alpha_2,...,\alpha_d \sim \mathcal{N}(0, 1)}\left[a_i \right] + \langle \vec{ v_i }, \m \rangle )  \right]
\end{equation}

\noindent We now consider different cases for $i$ to further simplify Equation~\eqref{eq:EMprojection}.
\begin{itemize}
  \item[--] When $i = 1$, we have that $ \langle \hat{\l}, \l' \rangle = \Exp_{y \sim \mathcal{N}(0, 1)}\left[ \tanh( \norm{\l} ( y + \langle \hat{\l}, \m \rangle ) ) \left( y + \langle \hat{\l}, \m \rangle \right) \right]$. This is equivalent with an iteration of EM in one dimension and thus from Theorem \ref{thm:singled} we get that
\begin{equation}
  \label{eq:firstCoordinateConv}
  |\langle \hat{\l}, \m \rangle - \langle \hat{\l}, \l' \rangle| \le \kappa |\langle \hat{\l}, \m \rangle - \langle \hat{\l}, \l \rangle|
\end{equation}
  where 
\[ \kappa = \exp\left( - \frac{\min(\langle \hat{\l}, \l \rangle, \langle \hat{\l}, \m \rangle)^2}{2} \right) = \exp\left( - \frac{\min(\langle {\l}, \l \rangle, \langle {\l}, \m \rangle)^2}{2 \langle {\l}, \l \rangle} \right) \]
  \item[--] When $i = 2$, $\Exp_{\alpha_2,...,\alpha_d \sim \mathcal{N}(0, 1)}\left[a_i \right] + \langle \vec{ v_i }, \m \rangle = \langle \hat{\l}^{\bot}, \m \rangle$ and thus 
  $$\langle \hat{\l}^{\bot}, \l' \rangle = \langle \hat{\l}^{\bot}, \m \rangle \Exp_{y \sim \mathcal{N}(0, 1)}\left[ \tanh( \norm{\l} ( y  + \langle \hat{\l}, \m \rangle) ) \right]$$
  Let $\kappa$ as defined before and using Lemma \ref{lem:tanhlemma} we get that
  \begin{equation}
    \label{eq:secondCoordinateConv}
    \langle \hat{\l}^{\bot}, \m \rangle \ge \langle \hat{\l}^{\bot}, \l' \rangle \ge ( 1 - \kappa ) \langle \hat{\l}^{\bot}, \m \rangle
  \end{equation}
  \item[--] When $i \ge 3$, $\Exp_{\alpha_2,...,\alpha_d \sim \mathcal{N}(0, 1)}\left[a_i \right] + \langle \vec{ v_i }, \m \rangle = 0$ and thus $\langle \vec{ v_i }, \l' \rangle = 0$.
\end{itemize}

\noindent  We can now bound the distance of $\l'$ from $\m$:
\begin{align*} 
  \norm{\l' - \m} &= \sqrt{ \sum_i \langle {\vec v_i}, \l' - \m \rangle^2 } = \sqrt{ \langle \hat{\l}, \l' - \m \rangle^2 + \langle \hat{\l}^{\bot}, \l' - \m \rangle^2 } \\
   &\stackrel{\text{(\ref{eq:firstCoordinateConv}), (\ref{eq:secondCoordinateConv})}}{\le} \sqrt{ \kappa^2 \langle \hat{\l}, \l - \m \rangle^2 + \kappa^2 \langle \hat{\l}^{\bot}, \l - \m \rangle^2 } \le \kappa \norm{\l - \m}
\end{align*}

We now have to prove that this convergence rate $\kappa$ decreases as the iterations increase. This is implied by the following lemmas which show that
$\min(\langle \hat{\l}, \l \rangle, \langle \hat{\l}, \m \rangle) \le \min(\langle \hat{\l}', \l' \rangle, \langle \hat{\l}', \m \rangle)$ 
\begin{lemma} \label{lem:multidProgress1}
  If $\norm{\l} \ge \langle \hat{\l}, \m \rangle$ then $\langle \hat{\l}, \m \rangle \le \norm{\l'}$ and $\langle \hat{\l}, \m \rangle \le \langle \hat{\l'}, \m \rangle$.
\end{lemma}

\begin{proof}
The analysis above implies that $\l'$ can be written in the form $\l' = \alpha \cdot \hat{\l} + \beta \cdot \hat{\l}^{\bot}$, where $ \langle \hat{\l}, \m \rangle \le \alpha \le \norm{\l}$ and $0 \le \beta \le \langle \hat{\l}^{\bot}, \m \rangle$. 
It is easy to see that the first inequality holds since $\norm{\l'} \ge \alpha \ge \langle \hat{\l}, \m \rangle$.
For the second, we write $\langle \hat{\l'}, \m \rangle$ as:
\[ \langle \hat{\l'}, \m \rangle = 
\frac{\langle \l', \m \rangle}{\norm{\l'}} = 
\frac{ \alpha \langle \hat{\l}, \m \rangle + \beta \langle \hat{\l}^{\bot}, \m \rangle}{\sqrt{\alpha^2 + \beta^2}} = 
\langle \hat{\l}, \m \rangle \frac{ 1  + \frac{ \langle \hat{\l}^{\bot}, \m \rangle } { \langle \hat{\l}, \m \rangle } \frac {\beta} {\alpha} }{\sqrt{1 + \left( \frac{\beta}{\alpha} \right)^2}} \ge
\langle \hat{\l}, \m \rangle \frac{ 1  + \left( \frac{\beta}{\alpha} \right)^2 }{\sqrt{1 + \left( \frac{\beta}{\alpha} \right)^2}} \ge \langle \hat{\l}, \m \rangle
\]
where we used the fact that $\frac{ \langle \hat{\l}^{\bot}, \m \rangle } { \langle \hat{\l}, \m \rangle } \ge  \frac{\beta} {\alpha}$ which follows by the bounds on $\alpha$ and $\beta$.
\end{proof}

\begin{lemma} \label{lem:multidProgress2}
  If $\norm{\l} \le \langle \hat{\l}, \m \rangle$ then $\norm{\l} \le \norm{\l'} \le \langle \hat{\l'}, \m \rangle$.
\end{lemma}
\begin{proof}
We have that $\l' = \alpha \cdot \hat{\l} + \beta \cdot \hat{\l}^{\bot}$, where $\norm{\l} \le \alpha \le \langle \hat{\l}, \m \rangle$ and $0 \le \beta \le \langle \hat{\l}^{\bot}, \m \rangle$.
We also have $\langle \l', \m \rangle = \alpha \langle \hat{\l}, \m \rangle + \beta \langle \hat{\l}^{\bot}, \m \rangle \ge \alpha^2+\beta^2 = \norm{\l'}^2 \ge \alpha^2 \ge \norm{\l}^2$ so the lemma follows.
\end{proof}

Finally substituting back in the basis that we started before changing coordinates to make the covariance matrix identity we get the result as stated at the theorem.
\end{proof}
 
  \section{An Illustration of the Speed of Convergence}
\label{sec:quan}

Using our results in the previous sections we can calculate explicit speeds of convergence of EM to its fixed points. In this section, we present some results with this flavor. For simplicity, we start with single dimensional case, and discuss the multi-dimensional case in the end of this section.

Let us consider a mixture of two single-dimensional Gaussians whose signal-to-noise ratio $\eta = \mu/\sigma$ is equal to $1$. There is nothing special about the value of $1$, except that it is a difficult case to consider since the Gaussian components are not separated, as shown in Figure~\ref{fig:SNR1}.
\begin{figure}[h!]
\centering
\includegraphics[width=7cm]{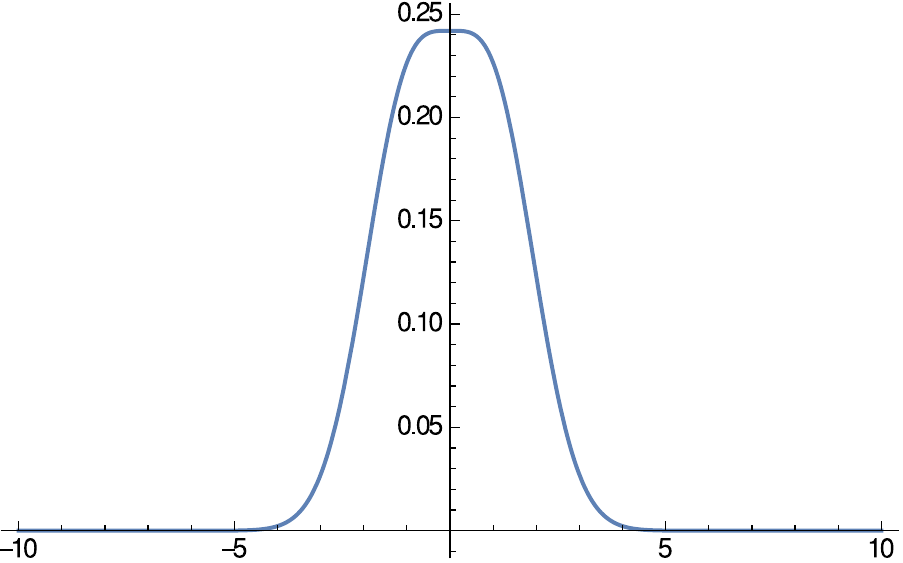}
\caption{The density of ${1 \over 2} {\cal N}(x; 1,1)+{1 \over 2} {\cal N}(x; -1,1)$.}
\label{fig:SNR1}
\end{figure}
When the SNR is larger, the numbers presented below still hold and in reality the convergence is even faster. When the SNR is even smaller than one, the numbers change, but gracefully, and they can be calculated in a similar fashion. 

We will also assume a completely agnostic initialization of EM, setting $\lambda^{(0)} \rightarrow +\infty$.\footnote{In the multi-dimensional setting, this would corrrespond to a very large magnitude $\l^{(0)}$ chosen in a random direction.} To analyze the speed of convergence of EM to its fixed point $\mu$, we first make the observation that in one step we already get to 
	$\lambda^{(1)} \le \mu + \sigma \sqrt{2 \over \pi}$. To see this we can plug $\lambda^{(0)} \rightarrow \infty$ into equation (\ref{eq:EMiteration1D}) to get:
\[ \lambda^{(1)} = \Exp_{x \sim \mathcal{N}(\mu, \sigma^2)}\left[ \text{sign}(x) x \right] = \Exp_{x \sim \mathcal{N}(\mu, \sigma^2)} \left[ \abs{x} \right], \]
  which equals the mean of the Folded Normal Distribution. A well-known bound for this mean is $\mu + \sqrt{2 \over \pi} \sigma$. Therefore the distance from the true mean after one step is $\abs{\lambda^{(1)} - \mu} \le \sqrt{2 \over \pi} \sigma $.

Now, using Theorem \ref{thm:singled}, we conclude that in all subsequent steps the distance to $\mu$ shrinks by a factor of at least $e^{+1/2}$. This means that, if we want to estimate $\mu$ to within additive error $1 \% \sigma$, then we need to run EM for at most $\lceil 2 \cdot \ln 100 + \ln {2 \over \pi}\rceil = 9$ additional steps. Accounting for the first step, $10$ iterations of the EM algorithm in total suffice to get to within error $1\%$, even when our initial guess of the mean is infinitely away from the true value! 

In Figure~\ref{fig:multidconvrates} we illustrate the speed of convergence of EM as implied by Theorem~\ref{thm:multid} in multiple dimensions. The plot was generated for a Gaussian mixture with $\m=(2~2)$ and $\Sigma=I$, but the behavior illustrated in this figure is generic (up to a transformation of the space by $\Sigma^{-{1 \over 2}}$). As implied by Theorem~\ref{thm:multid}, the rate of convergence depends on the distance of $\l^{(t)}$ from the origin $\vec{0}$ and the angle $\langle \l^{(t)}, \m \rangle$. The figure shows the directions of the EM updates for every point, and the factor by which the distance to the fixed point decays, with deeper colors corresponding to faster decays. There are three fixed points. Any point that is equidistant from $\m$ and $-\m$ is updated to $\vec{0}$ in one step and stays there thereafter. Points that are closer to $\m$ are pushed towards $\m$, while points that are closer to $-\m$ are pushed towards $-\m$.

\begin{figure}[h!]
\centering
\includegraphics[width=7cm]{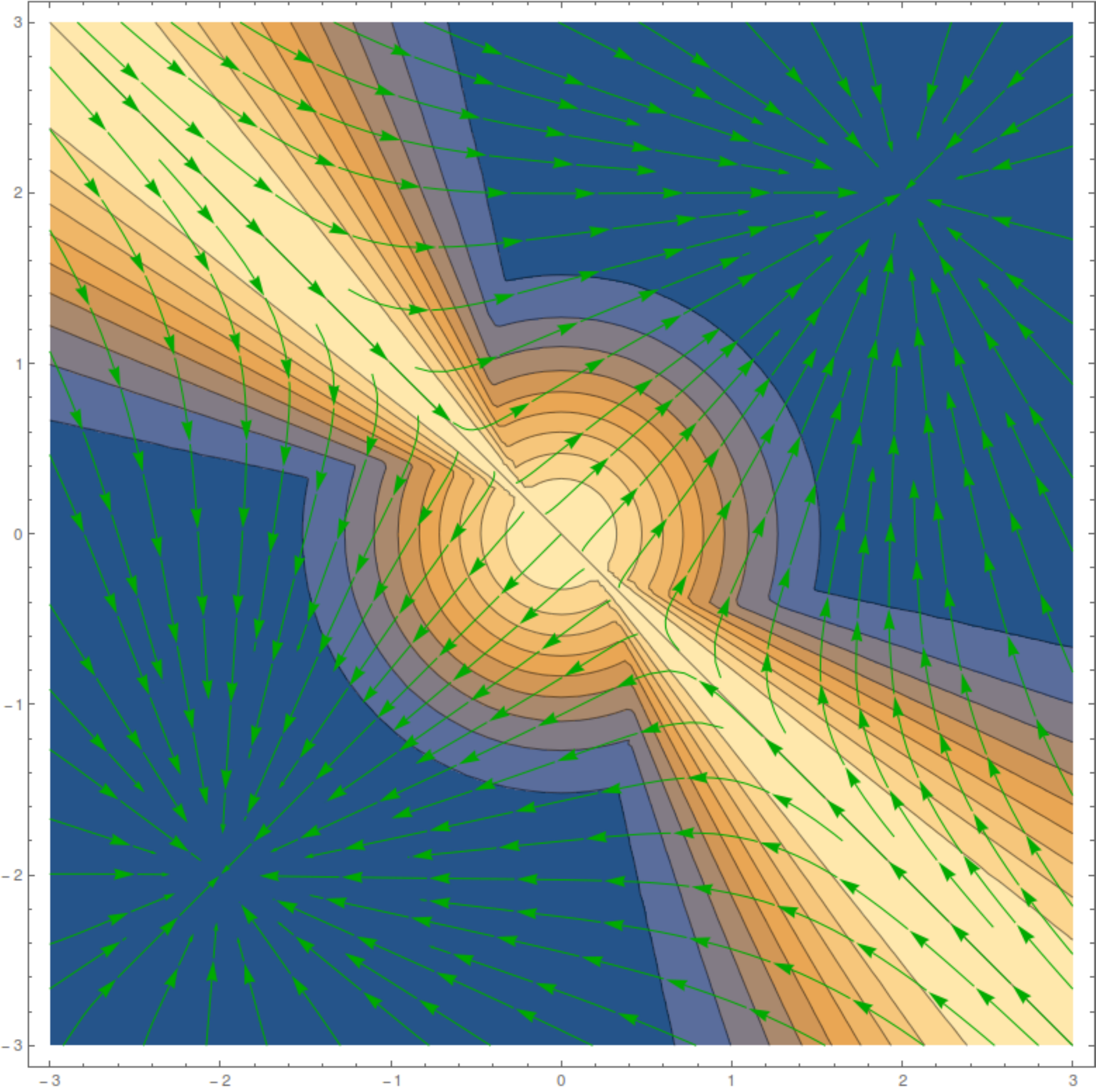}
\caption{Illustration of the Speed of Convergence of EM in Multiple Dimensions as Implied by Theorem~\ref{thm:multid}.}
\label{fig:multidconvrates}
\end{figure}

\begin{remark}[General Speed of Convergence] \label{remark:general speed of convergence}The analysis given above for $SNR=1$, generalizes to arbitrary $SNR$s, at the cost of a factor of $O(1/SNR^2)$ in the number of iterations. It also generalizes to obtain an arbitrary approximation $\epsilon$, at a cost of a factor of $O(\log 1/\epsilon)$. In multiple dimensions, we could run EM from a random initialization. The number of iterations for an approximation of $\epsilon$ in Mahalanobis distance would depend on the angle of the initial iterate with $\m$. Under a random initialization, the cosine of that angle is expected to be $\Theta(1/\sqrt{d})$, resulting in a total $O(d/SNR^2 \cdot \log(1/\epsilon))$ number of iterations. We show that we can boostrap EM to obtain a better initialization, starting from a random one, improving the angle to $\Omega(1)$, after $O(\log d/SNR^2)$ iterations. With a constant angle, EM takes $O(1/SNR^2 \cdot \log(1/\epsilon))$ iterations to give $\epsilon$ error, as in the single dimension, overall improving exponentially the dependence on $d$. We describe our bootstrapping operation in the context of our analysis of the finite sample EM in Section~\ref{sec:initialization}.
\end{remark}

  \section{Sample Based Model}
\label{sec:sample}

The main goal of this section is to prove convergence guarantees for the EM algorithm, when we have a finite sample. Similarly to the Section \ref{sec:multid} we willquantify our approximation guarantees using the Mahalanobis distance $\norm{\cdot}_{\Sigma}$ between two vectors with respect to matrix $\Sigma$, which we remind is defined as follows:
  \[ \norm{\vec{x} - \vec{y}}_{\Sigma} = \sqrt{(\vec{x} - \vec{y})^T \Sigma^{-1} (\vec{x} - \vec{y})}. \]
\noindent Also for the simplicity of the notation, it is useful to define the 
\textit{Mahalanobis inner product} between two vectors with respect to matric $\Sigma$ as follows:
	\[ \langle \vec{x}, \vec{y} \rangle_{\Sigma} = \vec{x}^T \Sigma^{-1} \vec{y}. \]
\noindent Towards our goal, we encounter two challenges. 

The first is that we cannot assume that we exactly know the mean of the mixture distribution $(\m_1 + \m_2)/2$. Our only access to this mean is via samples. We therefore use samples to estimate it. Then, we translate the origin to our estimate, and write the EM iteration for finding the mean of one of the two mixture components with respect to this origin. Given the error incurred in the approximation of $(\vec{\mu}_1 + \vec{\mu}_2)/2$, we propose to stabilize the sample-based EM iteration by including in the sample for each sampled point $\vx_i$ its symmetric point $-\vx_i$. This is the sample based version that we analyze, although our analysis goes through without this stabilization.

The other challenge has to do with the speed of convergence, as discussed in Remark~\ref{remark:general speed of convergence}.  Recall that the convergence of EM is a function of the angle $\langle \hat{\l}^{(0)}, \hat{\m} \rangle_{\Sigma}$, where $\hat{\l}^{(0)}$ is the unit vector in the direction of ${\l}^{(0)}$. In high dimensions
choosing a random  starting point $\l^{(0)}$ leads to an inner product that has value approximately $1/\sqrt{d}$. Thus the first step is to find a starting
point $\l^{(0)}$ such that the mean vector $\m$ (after the translation by the estimation of $(\m_1 + \m_2)/2$) has enough projection in the direction of $\l^{(0)}$. How can we do this? We actually bootstrap the EM algorithm to get such a good initialization. We show that, if we run EM starting from a random vector with small $\ell_2$ norm, then the EM algorithm will output a vector $\l$ with
$\langle \hat{\l}, \hat{\m} \rangle_{\Sigma} \ge 1/2$ after $O(\log d/SNR^2)$ iterations, where $SNR$ is defined as $SNR = \norm{\m}_{\Sigma}$.

  At this point we multiply $\l$ with a large positive constant $M$ and we continue the EM iteration with $\l^{(0)} = M\l$ using fresh samples for any iteration from
now on. This stage needs only logarithmic number of steps with respect to $1/\eps$, $d$ and polynomial in $1/SNR$. 
At each of these steps we prove that the sample based iteration is very well concentrated around its expectation and thus after few steps and $\tilde{O}((d / \eps^2) \poly(1/SNR) \log 1/\eta)$ 
samples we will find an estimation $\l$ such that
\[ \Prob\left( \norm{\l - \m}_{\Sigma} \ge \eps \right) \le \eta. \]

Our goal is to prove that $\norm{\l - \m}_{\Sigma} \le \eps$ holds with \textit{high probability}, which means 
$\eta = \poly\left( \frac{\eps^2}{d} \right)$. Also for any vector $\v \in \reals^d$ we use $\hat{\v}$ to refer
to the unit vector in the direction of $\v$. We present our results in the following order
\begin{enumerate}
  \item {\bf Centering:} Find an estimation $\vc$ of $(\m_1 + \m_2)/2$ using $\tilde{O}((d / \eps^2)) \cdot \log 1/\eta$ samples such that
        \[ \Prob\left(\norm{\vc - \frac{\m_1 + \m_2}{2}}_{\Sigma} \ge \eps \right) \le \eta / 3. \]
	We use $\vd$ to refer to the error in our estimation $\vc - \frac{\m_1 + \m_2}{2}$.
  \item {\bf Initialization:} We bootstrap EM to find a good initialization. In particular, starting from a randomly chosen vector, we run EM for $O(\log d/SNR^2)$ iterations using $\tilde{O}((d / \eps^2)) \cdot \log 1/\eta$ samples to get a unit vector $\hat{\l}$ such that
        \[ \Prob\left(\langle \hat{\l}, \hat{\m} \rangle_{\Sigma} \ge 1/2 \right) \le \eta / 3. \]
  \item {\bf Main Execution:} Setting $\l^{(0)} = M \hat{\l}$ for some large constant $M$ and running EM for $t = O((1/SNR^2) \log(1/\eps))$ iterations, using  $\tilde{O}((d / \eps^2 SNR^4)) \cdot \log 1/\eta$ fresh
        samples at each iteration, we get a vector $\l^{(t)}$ such that
        \[ \Prob\left(\norm{\l^{(t)} - \m_1}_{\Sigma} \ge \eps \right) \le \eta / 3. \]
\end{enumerate}

Combining the above steps all together we get our main theorem for this section.
\begin{theorem} \label{thm:sampleBased}
    If $\eps \le SNR$ and using $\tilde{O}(d/\eps^2) \cdot \poly(1 / SNR) \cdot \log 1/\eta$ samples we get an estimation $\l^{(t)}$ such that there is a constant $\eta$
that satisfies
  \[ \Prob \left[ \norm{\l^{(t)} - \m_1}_{\Sigma} \ge \eps \right] \le \eta. \]
\end{theorem}

\noindent We start now proving the lemmas for each of the steps described above.

\subsection{Centering} 
\label{sec:estimating the mean}

\begin{lemma} \label{lem:meanEst}
  For any $\eta < 1$, using $\tilde O(d / \eps^2 \log(1/\eta))$ samples there exists an estimator $\vc$ of the mean $(\m_1 + \m_2)/2$ such that if $\vd = \vc - (\m_1 + \m_2)/2$ then

  \[ \Prob\left[ \norm{\vd}_{\Sigma} \ge \eps \right] \le \eta \]
\end{lemma}

\begin{proof}
  We start by making the transformation $x \mapsto \Sigma^{-1/2} x$ to the space, so that the covariance matrix is the identity $I$. Finally we will take back this transformation and the Euclidean norm 
	becomes
the corresponding Mahalanobis.

  We will get an estimate $\vc$ of the mean $(\m_1 + \m_2)/2$ by drawing $O(d / \eps^2)$ samples from the mixture and working in each axis direction separately. We will compute the estimate $\hat{c}_i$ in axis direction
$\vec{e}_i$ as the average of the first and third quartile of the empirical distribution given by the samples. It suffices to show that with high probability every quartile is at most $\frac{\eps}{\sqrt{d}}$ away from the true
quartile of the distribution $p_{\m_1, \m_2} \cdot \vec{e}_i$.

  Let $p_{-\mu,\mu}$ be the mixture of Gaussian distributions obtained by centering $p_{\m_1,\m_2} \cdot \vec{e}_i$ around $0$.

  The cumulative distribution $F$ of $p_{-\mu,\mu}$ is given by $F(x) = \frac 1 2 \Phi\left( x + \mu \right) + \frac 1 2 \Phi\left( x - \mu \right)$. By the DKW inequality~\cite{devroye2012combinatorial} with $n$ samples we have
that the empirical distribution $F_n$ satisfies:

  $$\Prob {\Bigl (}\sup _{{x\in {\mathbb  R}}}{\bigl |}F_{n}(x)-F(x){\bigr |}>\eps/\sqrt{d} {\Bigr )}\leq 2 e^{{-2n\eps^{2}/d}}$$

In particular, for $x$ such that $F_n(x) = 1/4$, with high probability
$F(x) = 1/4 \pm \eps/\sqrt{d}$. Moreover, for $x^*$ such that $F(x^*) = 1/4$, we have that $|x-x^*| \le \frac {\eps/\sqrt{d}} {\min_{\xi \in [x,x^*]} F'(\xi) }$ by the mean value theorem.

Since for $\xi \in [x,x^*]$, $F(\xi) \in [1/4 - \eps/\sqrt{d}, 1/4 + \eps/\sqrt{d}]$, this implies that $\Phi\left( \xi + \mu \right) \in [1/4 - \eps/\sqrt{d}, 1/2 + 2 \eps/\sqrt{d}]$ and thus
$F'(\xi) \ge \frac 1  {2 } \Phi' \left( \xi + \mu \right) $. But $\Phi'(z) \ge \frac 1 5$ when $\Phi(z) \in [1/5, 4/5]$ which shows that $|x - x^*| = O( \eps/\sqrt{d})$. Similarly, this holds for the 3rd quartile as well
and thus the same bound holds for the mean as well of the two quartiles with probability $1- 2 e^{{-2n\eps^{2}/d}}$. Setting $n = O(d \log d / \eps^2)$, we get that the bound is violated with probability $O(1/d)$. Taking a union bound
for all $d$ axis directions, we get that for all $i$, $|\vd_i| = O(\eps/\sqrt{d})$ with constant probability. This implies that $\norm{\vd} \le \eps$ with constant probability. By using a factor of $\log 1/\eta$ more samples it is easy to see that the error probability reduces to $\eta$.
\end{proof}

\subsection{Sample Based EM Iteration}

  Using the estimation of the center $\vc$ that we found in the previous section, we translate all our data and 
parameters so that $\vc \mapsto \vec{0}$. After this centering, the parameters $\m_1$, $\m_2$ become 
$\m + \vd$, $-\m + \vd$ and the covariance matrix remains the same.

  We will use $\tilde{\l}^{(t)}$ to refer to the estimation of $\m$ after $t$ steps of finite sample stabilized 
EM Iteration. By stabilized we mean that for any sample $\vx_i$ that we get, we include in our data set the
vector $-\vx_i$ too. This way one EM iteration is simpler and easier to analyze but the results hold even 
without this stabilization. Following exactly the same steps as in Section \ref{sec:model} we can see that
\begin{align} \label{eq:SBsampleEMiteration}
    \tilde{\l}^{(t + 1)} = \frac{1}{n} \sum_{i = 1}^n \tanh\left(\langle \tilde{\l}^{(t)}, \vx_i \rangle_{\Sigma}\right) \vx_i.
\end{align}

\subsection{Initialization of EM} \label{sec:initialization}

  We rewrite the sample based EM iteration is the following form
\begin{align*}
    \tilde{\l}^{(t + 1)} = \frac{1}{n} \sum_{i = 1}^n \tanh\left(\langle \tilde{\l}^{(t)}, \vx_i \rangle_{\Sigma}\right) \vx_i = \sum_{i = 1}^n \tanh\left(\tilde{\lambda}^{(t)} \langle \hat{\tilde{\l}}^{(t)}, \vx_i \rangle_{\Sigma}\right) \vx_i
\end{align*}

\noindent where $\tilde{\lambda}^{(t)} = \norm{\tilde{\l}^{(t)}}_{\Sigma}$ and $\hat{\tilde{\l}}^{(t)} = \frac{\tilde{\l}^{(t)}}{\norm{\tilde{\l}^{(t)}}_{\Sigma}} $ is the unit vector in the direction of $\tilde{\l}^{(t)}$.

  The basic idea of bootstraping that we describe in this section, is that if $\tilde{\lambda}^{(t)}$ is very 
small then the $\tanh$ function is very close to be linear. This linear approximation of $\tanh$ gives the 
following approximate form of the EM update
\begin{align} \label{eq:SBEMPowerCalc1}
    \tilde{\l}^{(t + 1)} \approx \frac{1}{n} \sum_{i = 1}^n \langle \tilde{\l}^{(t)}, \vx_i \rangle_{\Sigma} \vx_i = \left( \frac{1}{n} \sum_{i = 1}^n \vx_i \vx_i^T \right) \Sigma^{-1} \tilde{\l}^{(t)} = \hat{\Sigma}_p \Sigma^{-1} \tilde{\l}^{(t)}
\end{align}

\noindent where $\hat{\Sigma}_p$ is the empirical covariance matrix of the mixture $p_{\m + \vd, - \m + \vd }$. 
Now the intuition suggests that the direction of the maximum eigenvector of the matrix 
$\hat{\Sigma}_p \Sigma^{-1}$ is a direction that is spanned by $\m + \vd$, $-\m + \vd$ and hence a direction 
with small angle to the direction of $\m$. Also observe that this approximate EM iteration is actually an 
iteration of the power method for the matrix $\hat{\Sigma}_p \Sigma^{-1}$! Because of the efficiency of the 
power method, we expect that after a few steps this iteration will find a direction $\hat{\l}$ such that the 
inner product $\langle \hat{\m}, \hat{\l} \rangle$ is large enough. This direction is a good initialization for 
the EM algorithm as we will see in the next section. We now formaly demonstrate the intuition we described for
the approximate EM update.

\bigskip
  We start by bounding the error we introduce by replacing the $\tanh$ function with its linear approximation. 
It is very easy to see that $\abs{\tanh''(x)} \le 1$. Now by Taylor expansion of $\tanh$ around $0$ we get that
  \[ \abs{\tanh(x) - x} \le \max_{\xi} \left( \tanh''(\xi) \right) \frac{x^2}{2} \le \frac{x^2}{2}. \]

  We want $\tilde{\lambda}^{(t)}$ to be small enough such that the linear approximation of $\tanh$ is a good approximation. For this reason we pick for $\epsilon < 1$
	\begin{equation} \label{eq:Initial Conditions}
  	\tilde{\lambda}^{(0)} = \sqrt{\frac{2}{\mathcal{S}}} \cdot \epsilon \text{~~~~ where ~~~~} \mathcal{S} = \sum_{i = 1}^n \norm{\vx_i}^3_{\Sigma}.
	\end{equation}

  \noindent Also we choose the direction $\hat{\tilde{\l}}^{(0)}$ uniformly from the unit sphere. Also let's 
assume that after every step of EM we normalize the $\tilde{\l}^{(t)}$ to ensure that $\tilde{\lambda}^{(t)}$ 
satisfies \eqref{eq:Initial Conditions}. This renormalization is non-necessary and it could be easily dropped 
by choosing $\tilde{\lambda}^{(0)}$ to be so small that after $O(\log d / SNR^2)$ steps 
\eqref{eq:Initial Conditions} is still satisfied. Given \eqref{eq:Initial Conditions} we have that

\begin{align*}
    \norm{ \tilde{\l}^{(t + 1)} - \frac{1}{n} \sum_{i = 1}^n \langle \tilde{\l}^{(t)}, \vx_i \rangle_{\Sigma} \vx_i}_{\Sigma} & \le \frac{1}{n} \sum_{i = 1}^n \abs{\tanh\left(\tilde{\lambda}^{(t)} \langle \hat{\tilde{\l}}^{(t)}, \vx_i \rangle_{\Sigma}\right) - \langle \tilde{\l}^{(t)}, \vx_i \rangle_{\Sigma}} \norm{\vx_i}_{\Sigma} \\
    & \le \frac{1}{2n} \sum_{i = 1}^n \langle \tilde{\l}^{(t)}, \vx_i \rangle_{\Sigma}^2 \norm{\vx_i}_{\Sigma} \\
    & \le \frac{1}{2n} \left( \tilde{\lambda}^{(t)} \right)^2 \sum_{i = 1}^n \norm{\vx_i}_{\Sigma}^3 \implies
\end{align*}

\begin{align} \label{eq:SBEMPowerApprox}
  \norm{ \tilde{\l}^{(t + 1)} - \frac{1}{n} \sum_{i = 1}^n \langle \tilde{\l}^{(t)}, \vx_i \rangle_{\Sigma} \vx_i}_{\Sigma} & \le \frac{1}{n} \epsilon^2
\end{align}

At this point the calculations become much easier if we assume that we have already done the mapping $x \mapsto \Sigma^{- 1/2}x$ and when we are done we will take the inverse mapping and get the result in Mahalanobis distance. Equation \eqref{eq:SBEMPowerApprox} suggests that it suffices to analyze the convergence of the power method given by the following equation.
\begin{align} \label{eq:SBEMPowerIteration}
    \tilde{\t}^{(t + 1)} = \frac{1}{n} \sum_{i = 1}^n \langle \tilde{\t}^{(t)}, \vx_i \rangle \vx_i \overset{\eqref{eq:SBEMPowerCalc1}}{=} \hat{\Sigma}_p \tilde{\t}^{(t)}
\end{align}

\noindent where $\hat{\Sigma}_p$ is the empirical covariance matrix of the mixture $p_{\m + \vd, - \m + \vd }$. Before analyzing \eqref{eq:SBEMPowerIteration} lets see what happens if instead of the empirical covariance $\hat{\Sigma}_p$
we had the actual covariance $\Sigma_p$ of the mixture distribution $p_{\m + \vd, -\m + \vd}$. Then the iteration would be $\t^{(t + 1)} = \Sigma_p \t^{(t)}$. The covariance matrix of the mixture is
\begin{align} \label{eq:SBEMCovMixtureIdeal}
    \Sigma_p = I + \m \m^T.
\end{align}

\noindent Therefore the principal eigenvector of $\Sigma_p$ is $\hat{\m} = \frac{\m}{\norm{\m}}$ with eigenvalue $1 + \mu^2$, where $\mu = \norm{\m}$. All the other eigenvectors have eigenvalue $1$ and therefore the ratio of the
largest to the lowest eigenvalue is $\rho = 1 + \mu^2$.

To get the corresponding properties of $\hat{\Sigma}_p$ we observe that each $\vx_i$ can be written as
\[ \vx_i = \vy_i + z_i \vd + \m \]
\noindent where $\vy_i$ is distributed as $\mathcal{N}(0, I)$ and $z_i$ is a Rademacher indicator variable that
shows whether $\vx_i$ is coming from the distribution $\mathcal{N}(\m + \vd, I)$ or $\mathcal{N}(\m - \vd, I)$. 
We notice that $\vy_i$ and $z_i$ are independent. We can now rewrite $\hat{\Sigma}_p$ in terms of $\vy_i$ and 
$z_i$.
\[ \hat{\Sigma}_p = \frac{1}{n} \sum_{i = 1}^n \vy_i \vy_i^T + \frac{1}{n} \sum_{i = 1}^n z_i \left( \vy_i \vd^T + \vd \vy_i^T \right) + \frac{1}{n} \sum_{i = 1}^n \left( \vy_i \m^T + \m \vy_i^T \right) + \frac{1}{n} \sum_{i = 1}^n \left( z_i \vd + \m \right) \left( z_i \vd + \m \right)^T \]

\noindent For simplicity we define
\begin{align*}
  \hat{\Sigma}_1 & = \frac{1}{n} \sum_{i = 1}^n \vy_i \vy_i^T \\
  \hat{\Sigma}_2 & = \frac{1}{n} \sum_{i = 1}^n z_i \left( \vy_i \vd^T + \vd \vy_i^T \right) \\
  \hat{\Sigma}_3 & = \frac{1}{n} \sum_{i = 1}^n \left( \vy_i \m^T + \m \vy_i^T \right) \\
  \hat{\Sigma}_4 & = \frac{1}{n} \sum_{i = 1}^n \left( z_i \vd + \m \right) \left( z_i \vd + \m \right)^T
\end{align*}

It is easy to see that using $n = \tilde{O}(d / \eps^2)$ samples, $\hat{\Sigma}_1$ satisfies the following
lemma.
\begin{lemma} \label{lem:sigma1Concentration}
    For $n = \tilde{O}(d / \eps^2)$ let $\hat{\Sigma}_1 = \frac{1}{n} \sum_{i = 1}^n \vy_i \vy_i^T$, where $\vy_i$ is drawn from $\mathcal{N}(0, I)$. For any direction $\hat{\v}$ in $\reals^d$. Then
  \[ \Prob\left( \abs{\hat{\v}^T \hat{\Sigma}_1 \hat{\v} - 1} > \eps^2 \right) \le \poly\left(\frac{\eps^2}{d}\right). \]
\end{lemma}

\noindent We describe now a sketch of the proof of Lemma \ref{lem:sigma1Concentration}, using the following 
 Lemma 25 from \cite{DaskalakisKT15}. 
\begin{lemma25} \label{lem:lemma25}
	 Let $\Sigma, \hat{\Sigma} \in \reals^{d \times d}$ be two symmetric, positive semi-definite matrices, and 
let $(\lambda_1, \v_1), \dots, (\lambda_d, \v_d)$ be the eigenvalue-eigenvector pairs of $\Sigma$. Suppose that 
\begin{itemize}
	\item For all $i \in \{1, \dots d\}$, $\abs{\left( \frac{\v_i}{\sqrt{\lambda_i}} \right)^T \left(\Sigma - \hat{\Sigma} \right) \left( \frac{\v_i}{\sqrt{\lambda_i}} \right)} \le \epsilon$,
	\item For all $i, j \in \{1, \dots d\}$, $\abs{\left( \frac{\v_i}{\sqrt{\lambda_i}} + \frac{\v_j}{\sqrt{\lambda_j}} \right)^T \left(\Sigma - \hat{\Sigma} \right) \left( \frac{\v_i}{\sqrt{\lambda_i}} + \frac{\v_j}{\sqrt{\lambda_j}} \right)} \le 4 \epsilon$.
\end{itemize}
Then for all $\vec{z} \in \reals^d$, 
$\abs{\vec{z}^T \left( \Sigma - \hat{\Sigma} \right) \vec{z}} \le 3 d \epsilon \vec{z}^T \Sigma \vec{z}$.
\end{lemma25}

The projection of $\vec{y}_i$ in each combination of two eigenvectors of $\Sigma$ is an one dimensional 
gaussian with variance $1$. Therefore using $\tilde{O}(d / \eps^2)$ samples we can estimate $\hat{\Sigma}_1$
in this direction with error at most $\eps/\sqrt{d}$ with high probability. Therefore by a union bound on these
fixed direction that depend only to $\Sigma_1$ we can satisfy the conditions of Lemma 25 of 
\cite{DaskalakisKT15}. Then from the implication of Lemma 25, Lemma \ref{lem:sigma1Concentration} follows.

\bigskip
	We continue with $\hat{\Sigma_3}$.
	
\begin{lemma} \label{lem:sigma3Concentration}
    For $n = \tilde{O}(d / \eps^2)$ let $\hat{\Sigma}_3 = \frac{1}{n} \sum_{i = 1}^n \left( \vy_i \m^T + \m \vy_i^T \right)$, where $\vy_i$ is drawn from $\mathcal{N}(0, I)$, $z_i$ is a uniform Rademacher random variable and. For any
  direction $\hat{\v}$ in $\reals^d$. Then
  \[ \Prob\left( \abs{\hat{\v}^T \hat{\Sigma}_3 \hat{\v}} > 2 \eps \abs{\hat{\v}^T \m} \right) \le \poly\left(\frac{\eps^2}{d}\right) \]
\end{lemma}
\begin{proof}
  Let $\hat{\v}$ an arbitrary direction in $\reals^d$. Then we have that
  \[ \abs{\hat{\v}^T \hat{\Sigma}_3 \hat{\v}} = \frac{2}{n} \abs{\sum_{i = 1}^n (\hat{\v}^T \vy_i) (\m^T \hat{\v})} \le 2 \abs{\hat{\v}^T \m} \abs{\hat{\v}^T \left( \frac{1}{n} \sum_{i = 1}^n \vy_i \right)} \le 2 \abs{\hat{\v}^T \m} \norm{\frac{1}{n} \sum_{i = 1}^n \vy_i} \]

  \noindent Now we consider the quantity $\vec{e}_j^T \left( \frac{1}{n} \sum_{i = 1}^n \vy_i \right)$, where
$\vec{e}_j$ is the unit $j$th vector. This is equivalent with having $\frac{1}{n} \sum_{i = 1}^n y_i$ where $y_i$ is drawn
from $\mathcal{N}(0, 1)$. So we have that
  \[ \Prob\left[ \abs{\vec{e}_j^T \left( \frac{1}{n} \sum_{i = 1}^n \vy_i \right)} \ge \frac{\eps}{\sqrt{d}} \right] \le \poly\left( \frac{\eps^2}{d} \right) \]

  \noindent Now by doing a union bound over all $\vec{e}_j$ we get that
  \[ \Prob\left[ \norm{\frac{1}{n} \sum_{i = 1}^n \vy_i} \ge \eps \right] \le \poly\left( \frac{\eps^2}{d} \right) \]

  \noindent Using these we can conclude that
  \[ \Prob\left( \abs{\hat{\v}^T \hat{\Sigma}_3 \hat{\v}} > 2 \eps \abs{\hat{\v}^T \m} \right) \le \poly\left(\frac{\eps^2}{d}\right) \]
\end{proof}

  For $\Sigma_2$, it is easy to observe that in 
$\hat{\Sigma}_2 = \frac{1}{n} \sum_{i = 1}^n z_i \left( \vy_i \vd^T + \vd \vy_i^T \right)$, $z_i$ are 
independent from $\vy_i$ and so the product $z_i \vy_i$ is a sample from standard multinormal distribution 
$\mathcal{N}(0, I)$. Therefore we can substitute $z_i \vy_i$ with just $\vy_i$. Now using exactly the same 
analysis as in Lemma \ref{lem:sigma3Concentration} we can prove the following lemma.
\begin{lemma} \label{lem:sigma2Concentration}
    For $n = \tilde{O}(d / \eps^2)$ let $\hat{\Sigma}_2 = \frac{1}{n} \sum_{i = 1}^n \frac{1}{n} \sum_{i = 1}^n z_i \left( \vy_i \vd^T + \vd \vy_i^T \right)$, where $\vy_i$ is drawn from $\mathcal{N}(0, I)$, $z_i$ is a uniform
Rademacher random variable and. For any direction $\hat{\v}$ in $\reals^d$ and any $\vd$ such that $\norm{\vd} \le \eps$. Then
  \[ \Prob\left( \abs{\hat{\v}^T \hat{\Sigma}_2 \hat{\v}} > 2 \eps^2 \right) \le \poly\left(\frac{\eps^2}{d}\right) \]
\end{lemma}

  For $\hat{\Sigma}_4$ we do straight forward calculations. For simplicity we use $\mu = \norm{\m}_{\Sigma}$. 
Let $\hat{\v}$ be an arbitrary direction in $\reals^d$.
\begin{align*}
  \hat{\v}^T \hat{\Sigma}_4 \hat{\v} = \frac{1}{n} \sum_{i = 1}^n \left( z_i \vd^T \hat{\v} + \m^T \hat{\v} \right)^2 \implies
\end{align*}

\begin{align} \label{eq:sigma4Concentration}
  (\abs{\m^T \hat{\v}} - \eps)^2 \le \hat{\v}^T \hat{\Sigma}_4 \hat{\v} \le (\abs{\m^T \hat{\v}} + \eps)^2
\end{align}

  \noindent Now we calculate the variance of the samples in the direction of $\m$. We have
  \[ \hat{\m}^T \hat{\Sigma}_p \hat{\m} = \hat{\m}^T \hat{\Sigma}_1 \hat{\m} + \hat{\m}^T \hat{\Sigma}_2 \hat{\m} + \hat{\m}^T \hat{\Sigma}_3 \hat{\m} + \hat{\m}^T \hat{\Sigma}_4 \hat{\m} \]
  \noindent Using Lemmas \ref{lem:sigma1Concentration}, \ref{lem:sigma2Concentration}, \ref{lem:sigma3Concentration} and \eqref{eq:sigma4Concentration} we have that with probability at least 
	$1 - \poly(\eps^2 / d)$
  \[ \hat{\m}^T \hat{\Sigma}_1 \hat{\m} \ge 1 - \eps^2 \]
  \[ \hat{\m}^T \hat{\Sigma}_2 \hat{\m} \ge - 2 \eps^2 \]
  \[ \hat{\m}^T \hat{\Sigma}_3 \hat{\m} \ge - 2 \eps \mu \]
  \[ \hat{\m}^T \hat{\Sigma}_4 \hat{\m} \ge (\mu - \eps)^2 \]
  \noindent and so
  \[ \hat{\m}^T \hat{\Sigma}_p \hat{\m} \ge 1 - 3 \eps^2 - 2 \eps \mu + (\mu - \eps)^2. \]

  \noindent Now let any other direction $\hat{\v}$ with $\hat{\v}^T \m = a \mu$, with $a \ge 0$. Again using 
Lemmas \ref{lem:sigma1Concentration}, \ref{lem:sigma2Concentration}, \ref{lem:sigma3Concentration} and
\eqref{eq:sigma4Concentration} we have that with probability at least $1 - \poly(\eps^2 / d)$
  \[ \hat{\v}^T \hat{\Sigma}_1 \hat{\v} \le 1 + \eps^2 \]
  \[ \hat{\v}^T \hat{\Sigma}_2 \hat{\v} \le 2 \eps^2 \]
  \[ \hat{\v}^T \hat{\Sigma}_3 \hat{\v} \le 2 a \eps \mu \]
  \[ \hat{\v}^T \hat{\Sigma}_4 \hat{\v} \le (a \mu + \eps)^2 \]
  \noindent and so
  \[ \hat{\v}^T \hat{\Sigma}_p \hat{\v} \le 1 + 3 \eps^2 + 2 a \eps \mu + (a \mu + \eps)^2. \]

 \noindent The principal eigenvector $\hat{\v}$ of $\hat{\Sigma}_p$ has to satisfy
  \[ \hat{\v}^T \hat{\Sigma}_p \hat{\v} \ge \hat{\m}^T \hat{\Sigma}_p \hat{\m} \implies \]
  \[ 6 \eps^2 + 4 a \eps \mu + (a^2 - 1) \mu^2 \ge 0 \]

 \noindent Now using an $\eps$ such that $\eps \le \mu / 10$ we have that the above implies $a \ge 3/4$. This
proves the following Proposition that we use in the analysis of the initialization step.
\begin{proposition} \label{prop:principalVectorSp}
  Let $\hat{\Sigma}_p$ be the empirical covariance matrix computed from $n = \tilde{O}(d / \eps^2)$ samples from the distribution $p_{\m + \vd, \m - \vd}$. Then the principal eigenvector $\hat{\v}$ of $\hat{\Sigma}_p$ satisfies
  \[ \Prob \left( \langle \hat{\v}, \hat{\m} \rangle < \frac{3}{4} \right) \le \poly\left( \frac{\eps^2}{d} \right). \]
\end{proposition}

  The last thing we need to prove to complete the analysis of the initialization is the gap between the first 
and the second eigenvalue of $\hat{\Sigma}_p$. We want to use this gap for the analysis of the convergence rate 
of the power method iteration that takes place in the first steps. Given Proposition 
ref{prop:principalVectorSp}, we have that any direction $\hat{v}$ except from the principal one, has 
$\langle \hat{\v}, \hat{\m} \rangle \le \frac{1}{4}$ with high probability. Let $\hat{\v}'$ be the eigenvector
that corresponds to the second maximum eigenvalue of $\hat{\Sigma}_p$. Let $a = \hat{\v}'^T \m$, using Lemmas
\ref{lem:sigma1Concentration}, \ref{lem:sigma2Concentration}, \ref{lem:sigma3Concentration} and 
\eqref{eq:sigma4Concentration} we have that
  \[ \hat{\v}'^T \hat{\Sigma}_1 \hat{\v}' \le 1 + \eps^2 \]
  \[ \hat{\v}'^T \hat{\Sigma}_2 \hat{\v}' \le 2 \eps^2 \]
  \[ \hat{\v}'^T \hat{\Sigma}_3 \hat{\v}' \le 2 a \eps \mu \le \frac{\eps \mu}{2} \]
  \[ \hat{\v}'^T \hat{\Sigma}_4 \hat{\v}' \le (a \mu + \eps)^2 \le \left( \frac{\mu}{4} + \eps \right)^2 \implies \]
  \[ \hat{\v}'^T \hat{\Sigma}_p \hat{\v}' \le 1 + 3 \eps^2 + \frac{1}{2} \eps \mu + \left(\frac{1}{4} \mu + \eps\right)^2. \]

  \noindent On the other hand based on the fact that $\hat{\v}^T \hat{\Sigma}_p \hat{\v} \ge \hat{\m}^T \hat{\Sigma}_p \hat{\m}$ we conclude that
  \[ \hat{\v}^T \hat{\Sigma}_p \hat{\v} \ge 1 - 3 \eps^2 - 2 \eps \mu + (\mu - \eps)^2. \]

  \noindent Using also the hypothesis that $0 \le \eps \le \mu / 10$ we get that
  \[ \frac{\hat{\v}^T \hat{\Sigma}_p \hat{\v}}{\hat{\v}'^T \hat{\Sigma}_p \hat{\v}'} \ge \frac{1 - 3 \eps^2 - 2 \eps \mu + (\mu - \eps)^2}{1 + 3 \eps^2 + \frac{1}{2} \eps \mu + \left(\frac{1}{4} \mu + \eps\right)^2} \ge \frac{1 + \frac{232}{400} \mu^2}{1 + \frac{101}{400} \mu^2} \ge \min\left\{1 + \frac{1}{4} \mu^2, 2\right\}. \]

Hence we get the following proposition.

\begin{proposition} \label{prop:firstSecondGap}	
	Let $\hat{\Sigma}_p$ be the empirical covariance matrix computed from $n = \tilde{O}(d / \eps^2)$ samples 
from the distribution $p_{\m + \vd, \m - \vd}$. Let also $\hat{\rho}$ be the ratio of the magnitude of the 
first two eigenvalues of $\hat{\Sigma}_p$ then
 \[ \Prob \left( \hat{\rho} < \min\left\{1 + \frac{1}{4} \mu^2, 2\right\} \right) \le \poly\left( \frac{\eps^2}{d} \right). \]
\end{proposition}

  It is well known and easy to prove that if we choose a random vector $\l^{(0)}$ uniformly from the half unit 
sphere, defined by $\hat{\m}$, then we will have that $\langle \l^{(0)}, \hat{\m} \rangle \ge 1 / \sqrt{d}$ 
with high probability.

  By standard analysis of the power method, Chapter 21.3 \cite{ShalevS14}, we know that the number of 
iterations we need to get within a constant angle from the principal eigenvector, starting from angle 
$1/\sqrt{d}$ is $O(\log d / \log \hat{\rho})$ where $\hat{\rho}$ is, as we have said, the ratio of the first 
two eigenvalues of $\hat{\Sigma}_p$. Therefore after $O(\log d / \log \hat{\rho}) = O(\log d / \min(\mu^2, 1))$ 
steps the iteration of $\tilde{\t}$ will find a vector $\t$ such that $\langle \hat{\tau}, \hat{\m} \rangle$ is 
at least $2/3$.

  Now we are ready to analyze the performance of $\tilde{\l}^{(t)}$ as be described in the beginning of the 
section. Applying \eqref{eq:SBEMPowerApprox} repeatedly at every iteration we get that after 
$t = O(\log d / \log \hat{\rho})$ iterations it holds that
\[ \norm{\frac{1}{n}\tilde{\l}^{(t)} - \tilde{\tau}^{(t)}} \le \frac{1}{n} \epsilon^2 \left( 1 + \hat{\rho} + \hat{\rho}^2 + \cdots + \hat{\rho}^{k} \right) = \frac{1}{n} \epsilon^2 \frac{\hat{\rho}^{k + 1} - 1}{\hat{\rho} - 1}. \]
\noindent But $k + 1$ is $O(\log d / \log \hat{\rho})$ and therefore
\[ \norm{\frac{1}{n}\tilde{\l}^{(t)} - \tilde{\tau}^{(t)}} \le \frac{1}{n} \epsilon^2 \frac{\poly(d) - 1}{\hat{\rho} - 1}. \]

\noindent Finally since $\hat{\rho}$ is $\min\left(1 + \frac{\mu^2}{4}, 2\right)$ and also $\eps \le \mu / 10$ 
we only need to set $\epsilon$ polynomially with respect to $\eps^2 / d$ and we will get that after the first
$O(\log d / \min(\mu^2, 1))$ iterations it holds that
\[ \norm{\frac{1}{n}\tilde{\l}^{(t)} - \tilde{\tau}^{(t)}} \le \eta \]
\noindent for some $\eta \le \eps / 6$. Which implies that
\[ \langle \hat{\tilde{\l}}, \hat{\m} \rangle \ge \frac{2}{3} - \frac{\eta}{\mu} \ge \frac{1}{2}. \]

\noindent This proves the following lemma and completes the proof of the initialization.

\begin{lemma}\label{lem:initialEM}
  Starting for a guess $\lambda^{(0)}$ such that 
$\lambda^{(0)} \le \sqrt{\frac{1}{18 \mathcal{S}}} \poly\left( \frac{\eps^2}{d} \right)$, with 
$\mathcal{S} = \sum_{i = 1}^n \norm{\vx_i}^3_{\Sigma}$ and after $O(\frac{\log d}{\eps^2})$ iterations of EM we
get a vector $\l$ such that
  \[ \langle \hat{\l}, \hat{\m} \rangle_{\Sigma} \ge \frac{1}{2}. \]
The probability of failure is at most $\poly \left( \frac{\eps^2}{d} \right)$.
\end{lemma}

\subsection{Finite-Sample EM Analysis} \label{sec:body of EM with samples}

We initialize EM at $\l^{(0)} = M \l$, where $\l$ is the point from Lemma~\ref{lem:initialEM},  for some large constant $M$.\footnote{It is easy to find such constant by getting a small number of samples and keeping the one that has maximum magnitude.} With this initialization, we run EM for $t = O((1/\mu^2) \log(d/\eps))$ steps using $O((d / \eps^2 \mu^4) \log 1/\delta)$ samples at each step, where for ease of notation we have set $\mu = SNR \equiv \norm{\m}_{\Sigma}$.

To study our sample-based EM iteration~\eqref{eq:SBsampleEMiteration} we will relate its progress to an appropriate population EM iteration. Note that this iteration differs from the population EM iteration that we discussed in  Section \ref{sec:model} and analyzed in Section~\ref{sec:multid}. The reason is that we have incurred an error $\vd$ in the estimation of the mean of the distribution in Section~\ref{sec:estimating the mean}. With respect to our estimated mean centering, the true means of the two Gaussian components are $\m + \vd$, $-\m + \vd$ rather than $\m$ and $-\m$. Another source of discrepancy comes from the fact that we included for each point $\vx_i$ in our sample its symmetric point $-\vx_i$.  
This implies that each
$\vx_i$ is coming with probability $1/2$ from the mixture $p_{\m + \vd, - \m - \vd}$ and with probability $1/2$ from the mixture $p_{\m - \vd, -\m + \vd}$. Given this, using again the same operations as in Section \ref{sec:model}, we
have that the corresponding population iteration, denoted by $\l^{(t)}$, is
\begin{align} \label{eq:SBpopulationEMiteration}
    \l^{(t + 1)} 
                 & = \frac{1}{2} \Exp_{\vx \sim \mathcal{N}(\m + \vd, \Sigma)}\left[ \tanh\left(\langle \l^{(t)}, \vx \rangle_{\Sigma}\right) \vx \right] + \frac{1}{2} \Exp_{\vx \sim \mathcal{N}(\m - \vd, \Sigma)}\left[ \tanh\left(\langle \l^{(t)}, \vx \rangle_{\Sigma}\right) \vx \right].
\end{align}
\noindent Our proof follows two steps illustrated in Figures~\ref{fig:graph2} and~\ref{fig:graph3}:
\begin{itemize}
\item {\bf Step 1:} First, we relate the population EM iteration defined by~\eqref{eq:SBpopulationEMiteration} to the vanilla population EM iteration defined by~\eqref{eq:EMiteration}; 
\item {\bf Step 2:} Then, we related the population EM iteration defined by~\eqref{eq:SBpopulationEMiteration} to the sample-based iteration.
\end{itemize}

\paragraph{Step 1:} To analyze the convergence of \eqref{eq:SBpopulationEMiteration}, we use Theorem \ref{thm:multid} for every component of the mixture. More precisely, let $\l_1^{(t)}$ and $\l_2^{(t)}$ be
\begin{align*}
  \l_1^{(t + 1)} & = \Exp_{\vx \sim \mathcal{N}(\m + \vd, \Sigma)}\left[ \tanh(\langle \l^{(t)}, \vx \rangle_{\Sigma}) \vx \right] \\
  \l_2^{(t + 1)} & = \Exp_{\vx \sim \mathcal{N}(\m - \vd, \Sigma)}\left[ \tanh(\langle \l^{(t)}, \vx \rangle_{\Sigma}) \vx \right].
\end{align*}

\noindent We know from Theorem \ref{thm:multid} that
\begin{align*}
  \norm{\l_1^{(t + 1)} - \m - \vd}_{\Sigma} & \le \kappa_1^{(t)} \norm{\l^{(t)} - \m - \vd}_{\Sigma} \\
  \norm{\l_2^{(t + 1)} - \m + \vd}_{\Sigma} & \le \kappa_2^{(t)} \norm{\l^{(t)} - \m + \vd}_{\Sigma}
\end{align*}

\noindent where
\begin{align*}
  \kappa_1^{(t)} & = \exp\left( - \frac { \min\left\{ \norm{\l^{(t)}}^2_{\Sigma}, \langle \m + \vd, \l^{(t)} \rangle_{\Sigma} \right\}^2 } { 2 \norm{\l^{(t)}}^2_{\Sigma} } \right) \\
  \kappa_2^{(t)} & = \exp\left( - \frac { \min\left\{ \norm{\l^{(t)}}^2_{\Sigma}, \langle \m - \vd, \l^{(t)} \rangle_{\Sigma} \right\}^2 } { 2 \norm{\l^{(t)}}^2_{\Sigma} } \right).
\end{align*}

\noindent But we have that
\begin{align*}
    \kappa_1^{(t)}, \kappa_2^{(t)} & \le \kappa'^{(t)} \triangleq \exp\left( - \frac { \min\left\{ \norm{\l^{(t)}}^2_{\Sigma}, \langle \m, \l^{(t)} \rangle_{\Sigma} \right\}^2 } { 2 \norm{\l^{(t)}}^2_{\Sigma} } + \frac{\eps^2}{2} \right)
\end{align*}

\noindent which implies that
\begin{align}
  \norm{\l_1^{(t + 1)} - \m - \vd}_{\Sigma} & \le \kappa'^{(t)} \norm{\l^{(t)} - \m - \vd}_{\Sigma} \label{eq:SBEMl1Contr} \\
  \norm{\l_2^{(t + 1)} - \m + \vd}_{\Sigma} & \le \kappa'^{(t)} \norm{\l^{(t)} - \m + \vd}_{\Sigma}. \label{eq:SBEMl2Contr}
\end{align}

\noindent We are ready now to bound the convergence of $\l^{(t)} = (\l_1^{(t)} + \l_2^{(t)}) / 2$
\begin{align*}
    \norm{\l^{(t + 1)} - \m}_{\Sigma} = & \norm{\frac{\l_1^{(t + 1)} + \l_2^{(t + 1)}}{2} - \m + \frac{\vd}{2} - \frac{\vd}{2}}_{\Sigma} \\
                                    \le & \frac{1}{2} \norm{\l_1^{(t + 1)} - \m - \vd}_{\Sigma} + \frac{1}{2} \norm{\l_2^{(t + 1)} - \m + \vd}_{\Sigma} \\
    \overset{\eqref{eq:SBEMl1Contr}, \eqref{eq:SBEMl2Contr}}{\le} & \frac{1}{2} \kappa'^{(t)} \left( \norm{\l^{(t)} - \m - \vd}_{\Sigma} + \norm{\l^{(t)} - \m + \vd}_{\Sigma} \right) \\
                                                                  & \le \kappa'^{(t)} \norm{\l^{(t)} - \m}_{\Sigma} + \kappa'^{(t)} \norm{\vd}_{\Sigma} \implies
\end{align*}

\begin{align} \label{eq:SBEMpopulContr}
    \norm{\l^{(t + 1)} - \m}_{\Sigma} \le \kappa'^{(t)} \norm{\l^{(t)} - \m}_{\Sigma} + \kappa'^{(t)} \norm{\vd}_{\Sigma}
\end{align}

  We can use the analysis of Section \ref{sec:multid} to see that $\kappa'^{(t)} \le \kappa'^{(0)}$. Therefore we also know that the population iteration satisfies:
\begin{align} \label{eq:SBEMpopulContr2}
    \norm{\l^{(t + 1)} - \m}_{\Sigma} \le \kappa \norm{\l^{(t)} - \m}_{\Sigma} + \kappa \norm{\vd}_{\Sigma}
\end{align}

\noindent where for simplicity we let $\kappa = \kappa'^{(0)}$. Now we have $\l^{(0)} = M \l$. Also $M$ is larger than $\mu$ and because of Lemma \ref{lem:initialEM} $\langle \hat{\l}, \hat{\m} \rangle_{\Sigma} \ge 1/2$ with high
probability, where $\mu = \norm{\m}_{\Sigma}$. Therefore
\[ \kappa \le \exp \left( - \frac{\mu^2}{4} + \frac{\eps^2}{2} \right) \]
\noindent and since $\eps \le \mu / 10$ we have that
\[ \kappa \le \exp \left( - \frac{\mu^2}{6} \right). \]
\noindent Also we have that $\norm{\vd}_{\Sigma} \le \eps$ and therefore \eqref{eq:SBEMpopulContr} becomes
\begin{align} \label{eq:SBEMpopulContr3}
    \norm{\l^{(t + 1)} - \m}_{\Sigma} \le e^{-\frac{\mu^2}{6}} \norm{\l^{(t)} - \m}_{\Sigma} + e^{-\frac{\mu^2}{6}} \eps.
\end{align}

\paragraph{Step 2:}  Our next goal is to show that the sample based iteration~\eqref{eq:SBsampleEMiteration} satisfies an equation similar
to \eqref{eq:SBEMpopulContr3}. We prove so by proving the concentration of  $\tilde{\l}^{(t)}$ around
its mean $\l^{(t)}$.

\begin{lemma} \label{lem:basicConcentration}
  Let $\tilde{\l}^{(t)} = \l^{(t)}$ then if we use 
$n = \tilde{O}\left(\frac{d}{\eps^2} \right)$ fresh samples at time step $t + 1$
we have that
\[ \Prob\left( \norm{\tilde{\l}^{(t + 1)} - \l}_{\Sigma} > \eps + \eps \cdot  \min\left\{ 1, \norm{\m}_{\Sigma} \right\} \cdot \norm{\l^{(t)} - \m}_{\Sigma} \right) \le \poly\left( \frac{\eps^2}{d} \right). \]
\end{lemma}

\begin{proof}
	Once again we assume that $\Sigma = I$ and for simplicity we set $\l = \l^{(t)}$, $\l' = \l^{(t + 1)}$ and
$\l = \tilde{\l} = \tilde{\l}^{(t)}$, $\tilde{\l}' = \tilde{\l}^{(t + 1)}$. Also we assume that are working on
a basis $\{\v_1, \v_2, \dots, \v_d\}$ such that all $\v_i$ for $i > 3$ are perpedicular to bot $\m$, $\l$ and
also $\v_2$ is perpedicular to $\l$ and $\v_1$ parallel to $\l$.

	\bigskip
	We first consider $i = 1$. In this case
	\[ \tilde{\lambda}'_1 = \frac{1}{n} \sum_{i = 1}^n \tanh(\langle \l, \vx_i\rangle) \langle \vx_i, \v_1 \rangle, \]
	\[ \lambda'_1 = \Exp[\tilde{\lambda}'_1] = \Exp_{\vx \sim p_{\m + \vd, \m - \vd}} \left[ \tanh(\langle \l, \vx\rangle) \langle \vx, \v_1 \rangle \right]. \]
	
\noindent Now we define $\mu_1 = \langle \m, \v_1 \rangle$, $\delta_1 = \langle \vd, \v_1 \rangle$ and we 
have
\[ \tilde{\lambda}'_1 = \frac{1}{n} \sum_{i = 1}^n \tanh( \lambda x_i ) x_i, \]
\[ \lambda'_1 = \Exp[\tilde{\lambda}'_1] = \Exp_{x \sim p_{\mu_1 + \delta_1, \mu_1 - \delta_1}} \left[ \tanh( \lambda x ) x \right] \]
\noindent where $x_i$ is distributed as $p_{\mu_1 + \delta_1, \mu_1 - \delta_1}$ and $\lambda = \norm{\l}$. For
simplicity we refer to $p_{\mu_1 + \delta_1, \mu_1 - \delta_1}$ as $\mathcal{D}_1$. Our goal is to bound the 
following probability
\begin{align} \label{eq:concLemmaToBound1}
	\Prob \left( \abs{\frac{1}{n} \sum_{i = 1}^n \tanh( \lambda x_i ) x_i - \Exp_{x \sim \mathcal{D}_1}[ \tanh( \lambda x) x ]} > \kappa \right)
\end{align}
\noindent to do so we use the general large deviation technique. Because of symmetry of $\mathcal{D}_1$, we 
have that the above probability is equal twice the probability
\[ \Prob \left( \exp \left( \theta \sum_{i = 1}^n \left( \tanh( \lambda x_i ) x_i - \Exp_{x \sim \mathcal{D}_1}[ \tanh( \lambda x) x ] \right) \right) > \exp \left( \theta n \kappa \right) \right). \]
\noindent Using Markov's inequality we get that
\begin{align} \label{eq:concLemmaToBound1h1}
  \Prob \left( \exp \left( \theta \sum_{i = 1}^n \left( \tanh( \lambda x_i ) x_i - \Exp_{x \sim \mathcal{D}_1}[ \tanh( \lambda x) x ] \right) \right) > \exp \left( \theta n \kappa \right) \right) \le \\
	\le \left( \frac{\Exp_{x \sim \mathcal{D}_1}\left[ \exp\left( \theta \left( \tanh(\lambda x) x - \Exp_{x \sim \mathcal{D}_1}\left[ \tanh(\lambda x) x \right] \right) \right) \right]}{\exp(\theta \kappa)} \right)^n
\end{align}
\noindent We therefore have to bound the quantity
\begin{equation} \label{eq:concLemmaToBound1h2}
\begin{split}
 \Exp_{x \sim \mathcal{D}_1}\left[ \exp\left( \theta\left( \tanh(\lambda x) x - \Exp_{x \sim \mathcal{D}_1}\left[ \tanh(\lambda x) x \right] \right) \right) \right] \le \\
 \le \Exp_{x \sim \mathcal{N}(\mu_1 + \delta_1, 1)}\left[ \exp\left( \theta \left( \tanh(\lambda x) x - \Exp_{x \sim \mathcal{N}(\mu_1 + \delta_1, 1)}\left[ \tanh(\lambda x) x \right] \right) \right) \right] \cdot \\
 	\cdot \exp\left( \frac{\theta}{2} \left( \Exp_{x \sim \mathcal{N}(\mu_1 + \delta_1, 1)}\left[ \tanh(\lambda x) x \right] - \Exp_{x \sim \mathcal{N}(\mu_1 - \delta_1, 1)}\left[ \tanh(\lambda x) x \right] \right) \right)
	\end{split}
\end{equation}

	The first term of \eqref{eq:concLemmaToBound1h2} is equal to
	\[ \Exp_{x \sim \mathcal{N}(0, 1)}\left[ \exp\left( \theta \left( \tanh(\lambda (x + \mu_1 + \delta_1)) (x + \mu_1 + \delta_1) - \Exp_{x \sim \mathcal{N}(0, 1)}\left[ \tanh(\lambda (x + \mu_1 + \delta_1)) (x + \mu_1 + \delta_1) \right] \right) \right) \right]. \]
\noindent Now we use the following Lemma 2.1 of \cite{Wainwright15}.
\begin{lemma2.1}
	Suppose that $f : \reals^d \to \reals^d$ is differentiable. Then for any convex function 
$\phi : \reals \to \reals$, we have
\[ \Exp\left[\phi\left( f(\vx) - \Exp\left[ f(\vx) \right] \right)\right] \le \Exp\left[ \phi\left( \frac{\pi}{2} \langle \nabla f(\vx), \vec{y} \rangle \right) \right]\]
where $\vx, \vec{y} \sim \mathcal{N}(0, I)$ are standard multivariate Gaussian, and independent.
\end{lemma2.1}

\noindent Combining this lemma with the fact that
\[ \frac{ \partial \tanh(\lambda (x + \mu_1 + \delta_1) ) (x + \mu_1 + \delta_1) }{\partial x} =  \tanh'(\lambda (x + \mu_1 + \delta_1) ) \lambda (x + \mu_1 + \delta_1) + \tanh(\lambda (x + \mu_1 + \delta_1) ) \le 2 \]
\noindent we get that
\[ \Exp_{x \sim \mathcal{N}(\mu_1 + \delta_1, 1)}\left[ \exp\left( \theta \left( \tanh(\lambda x) x - \Exp_{x \sim \mathcal{N}(\mu_1 + \delta_1, 1)}\left[ \tanh(\lambda x) x \right] \right) \right) \right] \le \exp\left( 5 \theta^2 \right).\]

\noindent Now for the second term of \eqref{eq:concLemmaToBound1h2} we notice that as we present in Section
\ref{sec:singled}
\[ \frac{\partial}{\partial \mu} \Exp_{x \sim \mathcal{N}(\mu, 1)}\left[ \tanh(\lambda x) x \right] = \Exp_{x \sim \mathcal{N}(\mu, 1)}\left[ \tanh'(\lambda x) \lambda x + \tanh(\lambda x) \right] \le 2 \]
\noindent which imlies using the mean value theorem that
\[ \exp\left( \frac{\theta}{2} \left( \Exp_{x \sim \mathcal{N}(\mu_1 + \delta_1, 1)}\left[ \tanh(\lambda x) x \right] - \Exp_{x \sim \mathcal{N}(\mu_1 - \delta_1, 1)}\left[ \tanh(\lambda x) x \right] \right) \right) \le \exp \left( \theta 2 \abs{\delta_1} \right). \]

\noindent Putting all together to \eqref{eq:concLemmaToBound1h2} we have that 

\[ \Exp_{x \sim \mathcal{D}_1}\left[ \exp\left( \theta\left( \tanh(\lambda x) x - \Exp_{x \sim \mathcal{D}_1}\left[ \tanh(\lambda x) x \right] \right) \right) \right] \le \exp \left( 5 \theta^2 + 2 \theta \abs{\delta_1} \right)\]

\noindent which implies that
\[ \Prob \left( \abs{\frac{1}{n} \sum_{i = 1}^n \tanh( \lambda x_i ) x_i - \Exp_{x \sim \mathcal{D}_1}[ \tanh( \lambda x) x ]} > \kappa \right) \le 2 \exp \left( n \left( 5 \theta^2 + 2 \theta \abs{\delta_1} - \theta \kappa \right) \right) \implies \]
\[ \Prob \left( \abs{\frac{1}{n} \sum_{i = 1}^n \tanh( \lambda x_i ) x_i - \Exp_{x \sim \mathcal{D}_1}[ \tanh( \lambda x) x ]} > \abs{\delta_1} + \tau \right) \le 2 \exp \left( - \frac{n \tau^2}{20} \right) \implies \]
\[ \Prob \left( \abs{\frac{1}{n} \sum_{i = 1}^n \tanh( \lambda x_i ) x_i - \Exp_{x \sim \mathcal{D}_1}[ \tanh( \lambda x) x ]} > \abs{\delta_1} + \frac{\eps}{\sqrt{d}} \right) \le 2 \exp \left( - \frac{n \eps^2}{20 d} \right) \]

\noindent Therefore with 
$n = O( \frac{d}{\eps^2} \log(d / \eps^2)) = \tilde{O}\left( \frac{d}{\eps^2} \right)$ we get
\begin{align} \label{eq:concBoundCoordinate1}
	\Prob \left( \abs{\frac{1}{n} \sum_{i = 1}^n \tanh( \lambda x_i ) x_i - \Exp_{x \sim \mathcal{D}_1}[ \tanh( \lambda x) x ]} > \abs{\delta_1} + \frac{\eps}{\sqrt{d}} \right) \le \poly \left( \frac{\eps^2}{d} \right)
\end{align}

	\bigskip
	We now consider $i = 2$. In this case
	\[ \tilde{\lambda}'_2 = \frac{1}{n} \sum_{i = 1}^n \tanh(\langle \l, \vx_i\rangle) \langle \vx_i, \v_2 \rangle, \]
	\[ \lambda'_2 = \Exp[\tilde{\lambda}'_2] = \Exp_{\vx \sim p_{\m + \vd, \m - \vd}} \left[ \tanh(\langle \l, \vx\rangle) \langle \vx, \v_2 \rangle \right]. \]
	
\noindent As before we define $\mu_2 = \langle \m, \v_2 \rangle$, $\delta_2 = \langle \vd, \v_2 \rangle$ and we 
have
\[ \tilde{\lambda}'_2 = \frac{1}{n} \sum_{i = 1}^n \tanh( \lambda x_i ) y_i, \]
\[ \lambda'_2 = \Exp[\tilde{\lambda}'_2] = \Exp_{x \sim p_{\mu_1 + \delta_1, \mu_1 - \delta_1}} \left[ \tanh( \lambda x ) \right] \Exp_{y \sim p_{\mu_2 + \delta_2, \mu_2 - \delta_2}} \left[ y \right] = \Exp_{x \sim \mathcal{D}_1} \left[ \tanh( \lambda x ) \right] \mu_2 \]
\noindent where $x_i$ is distributed as $p_{\mu_1 + \delta_1, \mu_1 - \delta_1}$, $y_i$ is distributed as 
$p_{\mu_2 + \delta_2, \mu_2 - \delta_2}$ and $\lambda = \norm{\l}$. For
simplicity we refer to $p_{\mu_2 + \delta_2, \mu_2 - \delta_2}$ as $\mathcal{D}_2$. Our goal is to bound the 
following probability
\begin{align} \label{eq:concLemmaToBound2}
	\Prob \left( \abs{\frac{1}{n} \sum_{i = 1}^n \tanh( \lambda x_i ) y_i - \Exp_{x \sim \mathcal{D}_1}[ \tanh( \lambda x) ] \mu_2} > \kappa \right)
\end{align}
\noindent to do so we use the general large deviation technique. Using the symmetry of $\mathcal{D}_1$ and 
$\mathcal{D}_2$ we have that the above probability is equal twice the
\[ \Prob \left( \exp \left( \theta \sum_{i = 1}^n \left( \tanh( \lambda x_i ) y_i - \Exp_{x \sim \mathcal{D}_1}[ \tanh( \lambda x) ] \mu_2 \right) \right) > \exp \left( \theta n \kappa \right) \right). \]
\noindent Using Markov's inequality we get that
\begin{align} \label{eq:concLemmaToBound2h1}
  \Prob \left( \exp \left( \theta \sum_{i = 1}^n \left( \tanh( \lambda x_i ) y_i - \Exp_{x \sim \mathcal{D}_1}[ \tanh( \lambda x) ] \mu_2 \right) \right) > \exp \left( \theta n \kappa \right) \right) \le \\
	\le \left( \frac{\Exp_{x \sim \mathcal{D}_1, y \sim \mathcal{D}_2}\left[ \exp\left( \theta \left( \tanh(\lambda x) y - \Exp_{x \sim \mathcal{D}_1}\left[ \tanh(\lambda x) \right] \mu_2 \right) \right) \right]}{\exp(\theta \kappa)} \right)^n
\end{align}

Using the fact that because of the initialization of EM $\lambda \ge \mu_1$ and by assumption $\delta_1 \le \mu_1$ and also let $\alpha = \min(1, \mu_1)$ we have to bound the quantity
\[ \Exp_{x \sim \mathcal{D}_1, y \sim \mathcal{D}_2}\left[ \exp\left( \theta\left( \tanh(\lambda x) y - \Exp_{x \sim \mathcal{D}_1}\left[ \tanh(\lambda x) \right] \mu_2 \right) \right) \right] \le \]
\[ \le \Exp_{x \sim \mathcal{N}(\mu_1 + \delta_1, 1), y \sim \mathcal{D}_2}\left[ \exp\left( \theta \left( \tanh(\lambda x) y - \Exp_{x \sim \mathcal{N}(\mu_1 + \delta_1, 1)}\left[ \tanh(\lambda x) \right] \mu_2 \right) \right) \right] \cdot \]
\[ \cdot \exp\left( \frac{\theta \mu_2}{2} \left( \Exp_{x \sim \mathcal{N}(\mu_1 + \delta_1, 1)}\left[ \tanh(\lambda x) \right] - \Exp_{x \sim \mathcal{N}(\mu_1 - \delta_1, 1)}\left[ \tanh(\lambda x) \right] \right) \right) \le \]
\[ \le \Exp_{x \sim \mathcal{N}(\mu_1 + \delta_1, 1), y \sim \mathcal{D}_2}\left[ \exp\left( \theta \left( \tanh(\lambda x) y - \Exp_{x \sim \mathcal{N}(\mu_1 + \delta_1, 1)}\left[ \tanh(\lambda x) \right] \mu_2 \right) \right) \right] \cdot \exp\left( \theta \abs{\mu_2} \abs{\delta_1} \right) \le \]
\[ \le \Exp_{x \sim \mathcal{N}(\mu_1 + \delta_1, 1), y \sim \mathcal{N}(\mu_2 + \delta_2, 1)}\left[ \exp\left( \theta \left( \tanh(\lambda x) y - \Exp_{x \sim \mathcal{N}(\mu_1 + \delta_1, 1)}\left[ \tanh(\lambda x) \right] (\mu_2 + \delta_2) \right) \right) \right] \cdot \]
\[ \cdot \exp\left( \theta \abs{\mu_2} \alpha \abs{\delta_1} \right) \cdot \exp\left( \theta \abs{\delta_2} \right) = \]
\[ = \Exp_{x \sim \mathcal{N}(\mu_1 + \delta_1, 1), y \sim \mathcal{N}(0, 1)}\left[ \exp\left( \theta \left( \tanh(\lambda x) (y + \mu_2 + \delta_2) - \Exp_{x \sim \mathcal{N}(\mu_1 + \delta_1, 1)}\left[ \tanh(\lambda x) \right] (\mu_2 + \delta_2) \right) \right) \right] \cdot \]
\[ \cdot \exp\left( \theta \abs{\mu_2} \alpha \abs{\delta_1} \right) \cdot \exp\left( \theta \abs{\delta_2} \right) = \]
\[ = \Exp_{x \sim \mathcal{N}(\mu_1 + \delta_1, 1), y \sim \mathcal{N}(0, 1)}\left[ \exp\left( \theta (\mu_2 + \delta_2) \left( \tanh(\lambda x) - \Exp_{x \sim \mathcal{N}(\mu_1 + \delta_1, 1)}\left[ \tanh(\lambda x) \right] \right) \right) \right] \cdot \]
\begin{equation} \label{eq:concLemmaToBound2h2}
 	\cdot \Exp_{x \sim \mathcal{N}(\mu_1 + \delta_1, 1), y \sim \mathcal{N}(0, 1)}\left[ \exp\left( \theta \tanh(\lambda x) y \right) \right] \cdot \exp\left( \theta \abs{\mu_2} \alpha \abs{\delta_1} \right) \cdot \exp\left( \theta \abs{\delta_2} \right).
\end{equation}

	The first term of \eqref{eq:concLemmaToBound2h2} is equal to
	\[ \Exp_{x \sim \mathcal{N}(\mu_1 + \delta_1, 1), y \sim \mathcal{N}(0, 1)}\left[ \exp\left( \theta (\mu_2 + \delta_2) \left( \tanh(\lambda x) - \Exp_{x \sim \mathcal{N}(\mu_1 + \delta_1, 1)}\left[ \tanh(\lambda x) \right] \right) \right) \right]. \]
	
\noindent Observe now that because of the initial conditions of EM at this step we have that $\mu_1 \ge \mu_2$
and $\lambda \ge \norm{\mu} \ge \mu_1$. Also it is not hard to prove that $\tanh(y^2) \ge 1 - \frac{1}{y}$. 
This means that if $\mu_2 \ge \poly \log(d / \eps^2)$ then with probability at least 
$\exp(-2 \poly \log(d / \eps^2))$ we will have that $\tanh(\lambda x) \ge 1 - \frac{1}{\mu_2}$. Now using the 
convexity of $\exp(\cdot)$ we have that the above term is less than
\[ \Exp_{x \sim \mathcal{N}(\mu_1 + \delta_1, 1)}\left[ \frac{\tanh(\lambda x) - \Exp[\tanh(\lambda x)] - 1 + \frac{1}{\mu_2}}{2} \exp(\theta (\mu_2 + \delta_2)) +\right. \]
\[ \left. + \frac{- \tanh(\lambda x) + \Exp[\tanh(\lambda x)] + 1}{2} \exp(- \theta (\mu_2 + \delta_2)) \right]. \]
\noindent Now using a simple Taylor expansion used in the proof of the Hoeffding bound we get that the first 
term of \eqref{eq:concLemmaToBound2h2} is less than or equal to
\[ \exp\left( \frac{\delta_2^2 \poly \log(d / \eps^2) \theta^2}{2} \right) \le \exp\left( \frac{\poly \log(d / \eps^2) \theta^2}{2} \right) \]
\noindent and this holds with high probability at least $\poly \left( \frac{\eps^2}{d} \right)$.

\noindent Now for the second term of \eqref{eq:concLemmaToBound2h2} we have that
\[ \Exp_{x \sim \mathcal{N}(\mu_1 + \delta_1, 1), y \sim \mathcal{N}(0, 1)}\left[ \exp\left( \theta \Exp_{x \sim \mathcal{N}(\mu_1 + \delta_1, 1)}\left[ \tanh(\lambda x) \right] y \right) \right] = \]
\[ = \exp\left( \frac{\theta^2}{2} \left( \Exp_{x \sim \mathcal{N}(\mu_1 + \delta_1, 1)}\left[ \tanh(\lambda x) \right] \right)^2 \right) \le \exp\left( \frac{\theta^2}{2} \right) \]

\noindent Putting all together to \eqref{eq:concLemmaToBound2h2} we have that 

\[ \Exp_{x \sim \mathcal{D}_1, y \sim \mathcal{D}_2}\left[ \exp\left( \theta\left( \tanh(\lambda x) y - \Exp_{x \sim \mathcal{D}_1}\left[ \tanh(\lambda x) \right] \mu_2 \right) \right) \right] \le \exp \left( \frac{1 + \poly \log(d / \eps^2)}{2} \theta^2 + (\abs{\mu_2} \alpha \abs{\delta_1} + \abs{\delta_2}) \theta \right)\]

\noindent which implies that
\[ \Prob \left( \abs{\frac{1}{n} \sum_{i = 1}^n \tanh( \lambda x_i ) y_i - \Exp_{x \sim \mathcal{D}_1}[ \tanh( \lambda x) ] \mu_2 } > \kappa \right) \le \]
\[ \le 2 \exp \left( n \left( \frac{1 + \poly \log(d / \eps^2)}{2} \theta^2 + (\abs{\mu_2} \alpha \abs{\delta_1} + \abs{\delta_2} - \kappa) \theta \right) \right) \implies \]
\[ \Prob \left( \abs{\frac{1}{n} \sum_{i = 1}^n \tanh( \lambda x_i ) y_i - \Exp_{x \sim \mathcal{D}_1}[ \tanh( \lambda x) ] \mu_2} > \eps \abs{\mu_2} \alpha + \abs{\delta_2} + \tau \right) \le 2 \exp \left( - \frac{n \tau^2}{4 \poly \log (d / \eps^2)} \right) \implies \]
\[ \Prob \left( \abs{\frac{1}{n} \sum_{i = 1}^n \tanh( \lambda x_i ) y_i - \Exp_{x \sim \mathcal{D}_1}[ \tanh( \lambda x) ] \mu_2} > \eps \abs{\mu_2} \alpha + \abs{\delta_2} + \frac{\eps}{\sqrt{d}} \right) \le 2 \exp \left( - \frac{n \eps^2}{4 d \poly \log (d / \eps^2)} \right) \]

\noindent Therefore with 
$n = 4 \frac{d}{\eps^2} \poly \log(d / \eps^2) = \tilde{O}\left( \frac{d}{\eps^2} \right)$ we get
\begin{align} \label{eq:concBoundCoordinate2}
	\Prob \left( \abs{\frac{1}{n} \sum_{i = 1}^n \tanh( \lambda x_i ) y_i - \Exp_{x \sim \mathcal{D}_1, y \sim \mathcal{D}_2}[ \tanh( \lambda x) y ]} > \eps \alpha \abs{\mu_2} + \abs{\delta_2} + \frac{\eps}{\sqrt{d}} \right) \le \poly \left( \frac{\eps^2}{d} \right)
\end{align}

\bigskip
	For any $i \ge 3$ we follow the same analysis as for the bound \eqref{eq:concBoundCoordinate2} but because of
the definition of the basis $\{\v_1, \dots, \v_d\}$ we have that $\mu_i = \langle \m, \v_i \rangle = 0$ and 
therefore 
\begin{align} \label{eq:concBoundCoordinate3}
	\Prob \left( \abs{\frac{1}{n} \sum_{i = 1}^n \tanh( \lambda x_i ) y_i - \Exp_{x \sim \mathcal{D}_1, y \sim \mathcal{D}_3}[ \tanh( \lambda x) y ]} > \abs{\delta_3} + \frac{\eps}{\sqrt{d}} \right) \le \poly \left( \frac{\eps^2}{d} \right)
\end{align}

	Finally if we combine \eqref{eq:concBoundCoordinate1}, \eqref{eq:concBoundCoordinate2} and
	\eqref{eq:concBoundCoordinate3} using the observation that $\norm{\l^{(t + 1)} - \m} \ge \abs{\mu_2}$ we get that

 \[ \Prob\left( \norm{\tilde{\l}^{(t + 1)} - \l} > \eps + \eps \min(\norm{\m}, 1) \norm{\l^{(t)} - \m} \right) \le \poly\left( \frac{\eps^2}{d} \right). \]
\end{proof}

\paragraph{Proof of Theorem~\ref{thm:sampleBased}:} Lemma \ref{lem:basicConcentration} and Equation \eqref{eq:SBEMpopulContr3} imply that, using $\tilde{O}(d/\eps^2 \mu^4)$ samples, we have:
\begin{align*}
    \norm{\tilde{\l}^{(t + 1)} - \m}_{\Sigma} \le \left(e^{-\frac{\mu^2}{6}} + \eps \min(\mu, 1) \right) \norm{\tilde{\l}^{(t)} - \m}_{\Sigma} + 2 \eps \mu^2.
\end{align*}

If $\mu > 1$ then 
\begin{align*}
  \left(e^{-\frac{\mu^2}{6}} + \eps \right) \le \frac{9}{10}.
\end{align*}

If $\mu \le 1$ then 
\begin{align*}
  \left(e^{-\frac{\mu^2}{6}} + \eps \mu \right) \le \left(e^{-\frac{\mu^2}{6}} + \frac{\mu^2}{20} \right) \le e^{-\frac{\mu^2}{10}}.
\end{align*}

These imply that
\begin{align} \label{eq:sampleContraction}
    \norm{\tilde{\l}^{(t + 1)} - \m}_{\Sigma} \le \max\left( e^{-\frac{\mu^2}{10}}, \frac{9}{10} \right) \norm{\tilde{\l}^{(t)} - \m}_{\Sigma} + 2 \eps \mu^2.
\end{align}

Therefore we need $\tilde{O}\left( \max \left( \frac{1}{\mu^2}, 1 \right) \log(1 / \eps)\right)$ steps in order to get error $3 \eps$. Since each step requires $\tilde{O}\left(\frac{d}{\eps^2 \mu^4}\right)$ samples,   Theorem \ref{thm:sampleBased} follows.
 
	\section*{Acknowledgements} We thank Sham Kakade for suggesting the problem to us, and for initial discussions. The authors were supported by NSF Awards CCF-0953960 (CAREER), CCF-1551875, CCF-1617730, and CCF-1650733, ONR Grant N00014-12-1-0999, and a Microsoft Faculty Fellowship.

	\bibliographystyle{alpha}
	\bibliography{ref}

\end{document}